%% file: main.tex
\documentclass[conference]{IEEEtran}
\usepackage{times}

\usepackage[numbers]{natbib}
\usepackage[bookmarks=true]{hyperref}

\usepackage{graphicx}
\usepackage{amsmath}
\usepackage{amssymb}
\usepackage{amsthm}
\usepackage{bm}

\begin{document}

% paper title
\title{A Biconvex Method for Minimum-Time Motion Planning Through Sequences of Convex Sets}

% You will get a Paper-ID when submitting a pdf file to the conference system
\author{
\authorblockN{Tobia Marcucci, Mathew Halm, William Yang, Dongchan Lee, and Andrew D. Marchese}
\authorblockA{Amazon Robotics \\ \texttt{\{tobmar,mshalm,yangwilm,ldc,andymar\}@amazon.com}}
}

\include{defs}

\maketitle

\begin{abstract}
We consider the problem of designing a smooth trajectory that traverses a sequence of convex sets in minimum time, while satisfying given velocity and acceleration constraints.
This problem is naturally formulated as a nonconvex program.
To solve it, we propose a biconvex method that quickly produces an initial trajectory and iteratively refines it by solving two convex subproblems in alternation.
This method is guaranteed to converge, returns a feasible trajectory even if stopped early, and does not require the selection of any line-search or trust-region parameter.
Exhaustive experiments show that our method finds high-quality trajectories in a fraction of the time of state-of-the-art solvers for nonconvex optimization.
In addition, it achieves runtimes comparable to industry-standard waypoint-based motion planners, while consistently designing lower-duration trajectories than existing optimization-based planners.
\end{abstract}

\IEEEpeerreviewmaketitle

\section{Introduction}
\label{sec:intro}

Selecting the most effective motion-planning algorithm for a robotic system often requires balancing three competing objectives: reliability, computational efficiency, and trajectory quality.
Consider Sparrow, the robot arm in Fig.~\ref{fig:robot} that sorts individual products into bins before they get packaged in the Amazon warehouses.
The algorithms that move Sparrow must be extremely reliable, as these robots handle millions of diverse products every day, and each failure requires expensive interventions.
They must be efficient, since every millisecond spent planning is taken away from other crucial computations, and limits the robot reactivity to sensor observations.
Finally, they should generate trajectories that push the robot to its physical limits, so that the work-cell throughput is maximized and the hardware is fully utilized.
Unfortunately, general-purpose methods for motion planning do not excel in all of these areas at once.

Sampling-based methods like PRM~\cite{kavraki1996probabilistic}, RRT~\cite{lavalle1998rapidly}, and their asymptotically optimal versions~\cite{karaman2011sampling} can be fast enough for real-time applications.
They are highly parallelizable~\cite{thomason2024motions} and can run on a GPU~\cite{bialkowski2011massively,pan2012gpu}.
They are also reliable in low-dimensional spaces, where dense sampling is computationally feasible. 
However, they become significantly less effective as the space dimension grows.
Additionally, although their kinodynamic variants support differential constraints~\cite{lavalle2001randomized,karaman2010optimal,li2016asymptotically}, sampling-based methods remain considerably less practical for designing smooth continuous trajectories than producing polygonal paths.

Trajectory-optimization methods based on nonconvex programming~\cite{betts2010practical,schulman2014motion} scale well to high-dimensional spaces and explicitly factor in the robot kinematics and dynamics.
Over the years, these techniques have become significantly faster~\cite{toussaint2014newton,howell2019altro} and, with the advent of specialized GPU implementations~\cite{sundaralingam2023curobo}, they are now even viable for real-time motion planning.
Despite these advances, the main limitation of trajectory optimization remains its reliance on local solvers, which require extensive parameter tuning, handcrafted warm starts, may suffer from inconsistent runtimes, and can even fail to find a solution.
While various strategies have been proposed to address these issues~\cite{tedrake2010lqr,kalakrishnan2011stomp,zucker2013chomp,ichnowski2020gomp}, trajectory optimization remains often too brittle for industrial deployment.

\begin{figure}[t]
\centering
\includegraphics[width=\columnwidth]{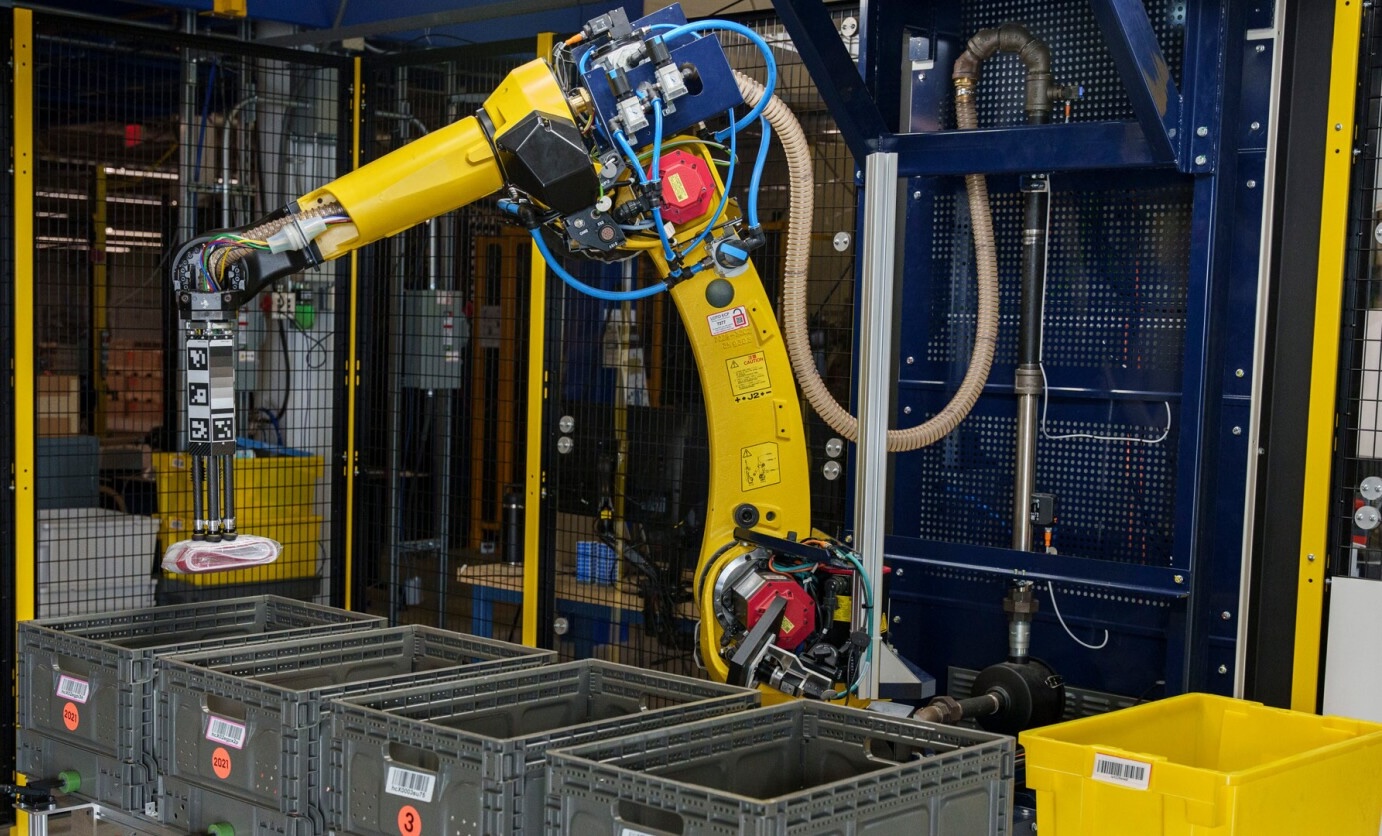}
\caption{Sparrow robot sorting products into bins in the Amazon warehouses.}
\label{fig:robot}
\end{figure}

Recently, a new family of motion planners that combine sampling-based and trajectory-optimization methods has stemmed from~\cite{deits2015efficient}.
First, the collision-free space is decomposed into safe convex sets.
This decomposition can be computed using region-inflation algorithms~\cite{deits2015computing,dai2024certified,werner2024faster,wang2024fast,wu2024optimal,werner2025superfast} and tailored sampling strategies~\cite{werner2024approximating}.
For UAVs, also safe flight corridors are widely used~\cite{chen2016online,liu2017planning,wu2024optimal}.
Then, the continuous trajectory is optimized in conjunction with the discrete sequence of sets to be traversed.
The work in~\cite{marcucci2023motion} has shown that, for a limited class of costs and constraints, this discrete-continuous problem is solvable through a single convex program, using the framework called Graphs of Convex Sets (GCS)~\cite{marcucci2024shortest,marcucci2024graphs}.
The extensions of GCS in~\cite{morozov2024multi,chia2024gcs} have enabled the solution of larger problems in a fraction of the time.
The motion planner in~\cite{marcucci2024fast} tackles a similar problem, but first selects a discrete sequence of safe sets using a heuristic, and later optimizes a continuous trajectory within these fixed sets.
This split sacrifices optimality but preserves completeness, and enables support for a broader range of costs and constraints.

This paper focuses on a problem similar to the one in the second phase of~\cite{marcucci2024fast}: we seek a trajectory that traverses a sequence of convex sets in minimum time, and satisfies convex velocity and acceleration constraints.
This is a purely continuous problem, but is nonconvex due to the joint optimization of the trajectory shape and timing.
Our contribution is a biconvex method, which we call Sequence of Convex Sets (SCS), that solves this problem effectively.
SCS starts by quickly producing a feasible trajectory.
Then, it alternates between two convex subproblems.
The first is obtained from the original nonconvex problem by fixing the points where the trajectory transitions from one safe set to the next. 
The second is derived similarly, by fixing the transition velocities.

As most multi-convex methods~\cite{shen2017disciplined}, SCS is heuristic: it typically finds high-quality trajectories quickly, but might not converge to the problem optimum, or within a given distance of it.
On the other hand, SCS is complete (i.e., guaranteed to find a feasible solution).
Its main algorithmic advantage is that the two convex subproblems are conservative approximations of the original nonconvex problem.
This allows us to take whole steps in the direction their optima without using a line search or a trust region, as done in~\cite{marcucci2024fast} and other trajectory-optimization methods~\cite{schulman2014motion,ichnowski2020gomp}.
This makes the convergence of SCS fast and monotone, and eliminates any parameter tuning.
Furthermore, it makes our algorithm anytime (it returns a feasible trajectory even if stopped early).

We show that SCS consistently finds high-quality trajectories in a fraction of the time of the state-of-the-art solvers \texttt{SNOPT}~\cite{gill2005snopt} and \texttt{IPOPT}~\cite{wachter2006implementation}.
We also demonstrate SCS on the task of transferring packages between bins using two Sparrow robots.
In this task, SCS designs lower-cost trajectories than the trust-region method from~\cite{marcucci2024fast}, and achieves runtimes comparable to waypoint-based methods that are commonly used in industry.

\subsection{Outline}

This paper is organized as follows.
In \S\ref{sec:statement}, we state our motion-planning problem and, in \S\ref{sec:bicvx}, we give a high-level overview of SCS.
The details on the two convex subproblems and the initialization step are illustrated in \S\ref{sec:fixed_velocities}, \S\ref{sec:fixed_points}, and \S\ref{sec:initialization}.
Up to this point, we work only with infinite-dimensional trajectories.
In \S\ref{sec:num}, we show how our method can be efficiently implemented on a computer by using piecewise B\'ezier curves as a finite-dimensional trajectory parameterization.
The strengths and limitations of SCS are discussed in \S\ref{sec:strengths} and \S\ref{sec:limitations}.
In \S\ref{sec:experiments}, we demonstrate the effectiveness of SCS through a variety of numerical experiments.

\subsection{Notation and convexity background}
\label{sec:notation}

In this paper, the variable $i$ is always understood to be a positive integer.
Thus, when saying for all $i \leq I$ we mean for all $i \in \{1,\ldots, I\}$.
Conversely, the variable $k$ is always nonnegative, and $k \leq K$ is shorthand for $k \in \{0,\ldots, K\}$.

We use calligraphic letters to represent sets, and bold letters for vectors, vector-valued functions, and matrices.
We use the notation $\lambda \cS = \{\lambda \bx : \bx \in \cS\}$ to denote the product of a scalar $\lambda \in \reals$ and a set $\cS \subseteq \reals^n$.
We will use multiple times the fact that if a set $\cS$ is convex then also the set $\{(\bx, \lambda): \lambda \geq 0 \ \bx \in \lambda \cS\}$ is convex~\cite[\S2.3.3]{boyd2004convex}.
The latter set is easily computed in practice: for instance, if $\cS$ is a polytope of the form $\{\bx: \bA \bx \leq \bb\}$, then the condition $\bx \in \lambda \cS$ is equivalent to $\bA \bx \leq \lambda \bb$.
A similar formula can be used for any set $\cS$ described in standard conic form~\cite[Ex.~4.3]{marcucci2024shortest}.

%We will use the following lemma multiple times.
%We omit its proof for brevity, and point the reader to~\cite[\S2.3.3]{boyd2004convex} for the proof of a similar result.
%
%\begin{lemma}
%\label{lem:perspective}
%%Let $f:\reals \rightarrow \reals$ be a function and $\cS \subseteq \reals^n$ a convex set.
%%The set
%%$$
%%\cT = \{(\bx, \lambda) \in \reals^{n+1}: f(\lambda) \geq 0, \ \bx \in f(\lambda) \cS\}.
%%$$
%%is convex if $f$ is affine, or $f$ is concave and $\cS$ contains the origin.
%Given a function $f:\reals \rightarrow \reals$ and a convex set $\cS \subseteq \reals^n$, consider the set
%$$
%\cT = \{(\bx, \lambda) \in \reals^{n+1}: f(\lambda) \geq 0, \ \bx \in f(\lambda) \cS\}.
%$$
%\begin{itemize}
%\item[(a)]
%If the function $f$ is affine, then the set $\cT$ is convex.
%\item[(b)]
%If the function $f$ is concave and the set $\cS$ contains the origin, then the set $\cT$ is convex.
%\end{itemize}
%\end{lemma}
%
%We note that the set $\cT$ is easily computed in practice.
%As an example, if $\cS$ is a polytope of the form $\{\bx: \bA \bx \leq \bb\}$, then the condition $\bx \in f(\lambda) \cS$ becomes simply $\bA \bx \leq f(\lambda) \bb$.
%A similar formula applies when $\cS$ is a generic convex set in standard conic form~\cite[Example~4.3]{marcucci2024shortest}.

\section{Problem Statement}
\label{sec:statement}

We seek a trajectory that traverses a sequence of convex sets in minimum time, subject to boundary conditions and convex velocity and acceleration constraints.
Fig.~\ref{fig:statement} shows a simple instance of this problem, and illustrates our notation.

\begin{figure}
\centering
\includegraphics[width=\columnwidth]{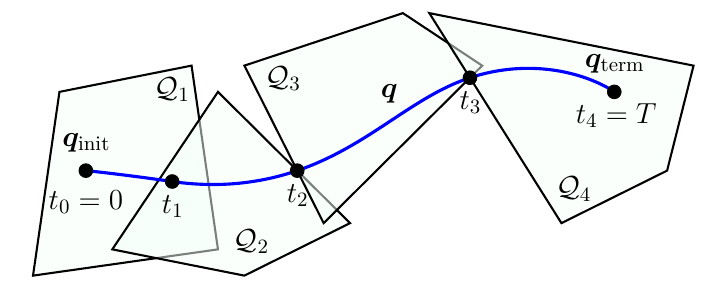}
\caption{Example of the motion-planning problem.
The safe convex sets to be traversed are $\cQ_1, \ldots, \cQ_4$.
The trajectory $\bq:[0,T] \rightarrow \reals^2$ is shown in blue.
The initial and terminal points are $\bq\init$ and $\bq\term$.
The times $t_0, \ldots, t_4$ determine the trajectory piece assigned to each safe set.}
\label{fig:statement}
\end{figure}

The sequence of safe convex sets is denoted as
$$
\cQ_1, \ldots, \cQ_I \subseteq \reals^n,
$$
where $I$ is the number of sets and $n$ is the space dimension.
%($I=4$ and $n=2$ in Fig.~\ref{fig:statement}).
Each set is assumed to be closed and intersect with the next:
\begin{align*}
& \cQ_i \cap \cQ_{i+1} \neq \emptyset,
&& i \irange{1}{I-1}.
\end{align*}
The trajectory is represented by the function $\bq : [0, T] \rightarrow \reals^n$, with time duration $T > 0$.
The initial $\bq(0)$ and terminal point $\bq(T)$ are fixed to $\bq\init \in \cQ_1$ and $\bq\term \in \cQ_I$, respectively.

The safe sets must be traversed in the given order, and no set can be skipped.
We denote with $t_1 \leq \ldots \leq t_{I-1}$ the \emph{transition times} at which the trajectory moves from one set to the next.
For simplicity of notation, we also let $t_0 = 0$ and $t_I = T$.
We then require that the trajectory $\bq(t)$ lie in the set $\cQ_i$ for all $t \in [t_{i-1}, t_i]$ and $i \irange{1}{I}$.
%This also ensures that the sequence of convex sets is traversed in the given order.

The trajectory velocity and the acceleration are denoted as $\dot \bq(t)$ and $\ddot \bq(t)$, respectively.
The first is assumed to be continuous, while the second is allowed to have discontinuities (i.e., $\bq$ is continuously differentiable).
The initial $\dot \bq(0)$ and terminal $\dot \bq(T)$ velocities are fixed to zero.
The trajectory derivatives must satisfy the constraints
$$
\dot \bq(t) \in \cV, \qquad \ddot \bq(t) \in \cA,
$$
at all times $t \in [0,T]$.
The sets $\cV$ and $\cA$ are closed and convex, and contain the origin in their interior:
$$
\bzero \in \interior(\cV), \quad \bzero \in \interior(\cA).
$$

Among the trajectories that verify the constraints above, we seek one of minimum time duration $T$.
This leads us to the following optimization problem:
\begin{subequations}
\label{eq:statement}
\begin{align}
\label{eq:statement_obj}
\minimize \quad & T \\
\subjectto \quad
\label{eq:statement_init}
& \bq(0) = \bq\init, \\
\label{eq:statement_term}
& \bq(T) = \bq\term, \\
\label{eq:statement_boundary_vel}
& \dot \bq(0) = \dot \bq(T) = \bzero, \\
\label{eq:statement_pos}
& \bq(t) \in \cQ_i, && t \in [t_{i-1}, t_i], \ i \irange{1}{I},\\
\label{eq:statement_vel}
& \dot \bq(t) \in \cV, && t \in [0, T], \\
\label{eq:statement_acc}
& \ddot \bq(t) \in \cA, && t \in [0, T], \\
\label{eq:statement_dur}
& t_i \leq t_{i+1}, && i \irange{1}{I-1}, \\
& t_0 = 0, \ t_I = T.
\end{align}
\end{subequations}
The variables are the trajectory $\bq$, the duration $T$, and the times $t_0, \ldots, t_I$.
The first makes the problem infinite dimensional.
The \emph{problem data} are the endpoints $\bq\init$ and $\bq\term$, the safe sets $\cQ_1, \ldots, \cQ_I$, and the constraint sets $\cV$ and $\cA$.
The differentiability constraint on the function $\bq$ is implicit here.

\subsection{Feasibility}

With the next proposition, we establish the feasibility of problem~\eqref{eq:statement}.
As in~\cite[\S{II}-C]{marcucci2024fast}, we do so by constructing a polygonal (i.e., piecewise linear) trajectory that satisfies all the problem constraints.
A similar construction will be used to initialize our method.

\begin{proposition}
\label{prop:feas}
If the problem data satisfy the assumptions listed above, then problem~\eqref{eq:statement} is feasible.
\end{proposition}

\begin{proof}
We construct a trajectory $\bq$ that starts at $\bq(t_0) = \bq\init$, terminates at $\bq(t_I) = \bq\term$, and interpolates any \emph{transition points} $\bq(t_i) \in \cQ_i \cap \cQ_{i+1}$ for $i \irange{1}{I-1}$.
For $i \irange{1}{I}$, the trajectory piece within the set $\cQ_i$ connects $\bq(t_{i-1})$ and $\bq(t_i)$ through a straight line.
Thus, the overall trajectory stays within the union of the safe sets and traverses them in the desired order.
We let the times $t_0, \ldots, t_I$ be well spaced, so that the velocity $\dot \bq$ and the acceleration $\ddot \bq$ can be small enough to lie in the constraint sets $\cV$ and $\cA$ at all times (recall that these sets contain the origin in their interior).
Finally, we require that the velocity be zero at each time $t_0, \ldots, t_I$.
This makes the velocity continuous, even though the trajectory is polygonal.
The resulting trajectory is feasible for problem~\eqref{eq:statement}.
\end{proof}

\subsection{Positive traversal times}

According to constraint~\eqref{eq:statement_dur}, the \emph{traversal time} $T_i = t_i - t_{i-1}$ of a safe set $\cQ_i$ can be zero.
This can be optimal if, e.g., a safe set is lower dimensional or our trajectory touches it only at an extreme point.
However, our biconvex method will assume that the traversal times $T_i$ are strictly positive, for all $i \irange{1}{I}$.
The following is a simple sufficient condition on the problem data that ensures this.
It forces our trajectory to cover a nonzero distance within each safe set.

\begin{assumption}
\label{ass:zero_time}
The boundary points and the safe sets are such that $\bq\init \notin \cQ_2$, $\bq\term \notin \cQ_{I-1}$, and
\begin{align*}
& \cQ_i \cap \cQ_{i+1} \cap \cQ_{i+2} = \emptyset,
&& i \irange{1}{I-2}.
\end{align*}
\end{assumption}

We will let this assumption hold throughout the paper, so that zero traversal times will always be infeasible in our optimization problems.
Alternatively, our algorithm can be easily modified to incorporate a small lower bound on the traversal times.
For most practical problems, this modification has a negligible effect on the optimal trajectories.

\section{Biconvex Method}
\label{sec:bicvx}

We give a high-level overview of our biconvex method here, deferring the details to later sections.

The observation at the core of SCS is that problem~\eqref{eq:statement} reduces to a convex program if we fix either the \emph{transition points} or the \emph{transition velocities}:
$$
\bq(t_1), \ldots, \bq(t_{I-1}),
\qquad 
\dot \bq(t_1), \ldots, \dot \bq(t_{I-1}).
$$
(More precisely, this is true modulo a small conservative approximation of the acceleration constraint~\eqref{eq:statement_acc}, which relies on an estimate of the traversal times.)
In these convex programs, the transition points or velocities are fixed, but the rest of the trajectory is optimized.

This observation motivates the method illustrated in Fig.~\ref{fig:algo} for solving problem~\eqref{eq:statement}:
\begin{itemize}
\item
\emph{Initialization (1st panel).}
We compute a polygonal trajectory that connects the initial $\bq\init$ and terminal point $\bq\term$, and has short time duration.
This is designed through a small number of convex programs.
\item
\emph{Fixed transition points (2nd panel).}
We fix the transition points $\bq(t_1), \ldots, \bq(t_{I-1})$ of the polygonal trajectory, and use its traversal times $T_1, \ldots, T_I$ to approximate the acceleration constraints.
This leads to a convex subproblem that improves the polygonal trajectory.
\item
\emph{Fixed transition velocities (3rd panel).}
We fix the transition velocities $\dot \bq(t_1), \ldots, \dot \bq(t_{I-1})$ of the improved trajectory, and use its traversal times $T_1, \ldots, T_I$ to approximate the acceleration constraints.
This leads to another convex subproblem that further improves our solution.
\item
\emph{Iterations (4th panel).}
We keep refining our trajectory by solving the two convex subproblems in alternation.
\item 
\emph{Termination (5th panel).}
We terminate when the relative objective decrease of an iteration is smaller than a fixed tolerance $\varepsilon \in (0,1]$.
The objective decrease is measured between any two consecutive subproblems of the same kind (fixed transition points or velocities).
\end{itemize}

\begin{figure}[t!]
\centering
\includegraphics[width=\columnwidth]{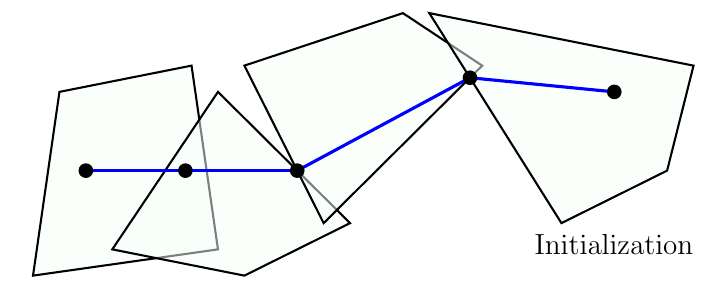} \\ \vspace{-1mm}
\includegraphics[width=\columnwidth]{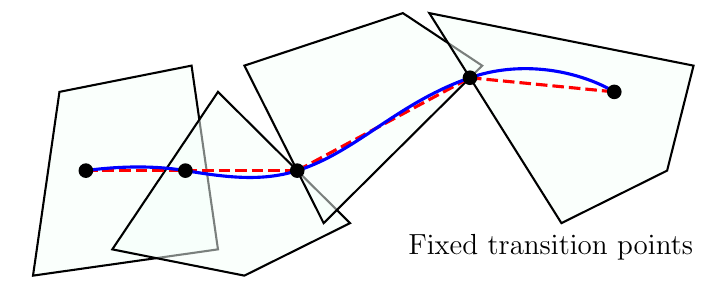} \\ \vspace{-1mm}
\includegraphics[width=\columnwidth]{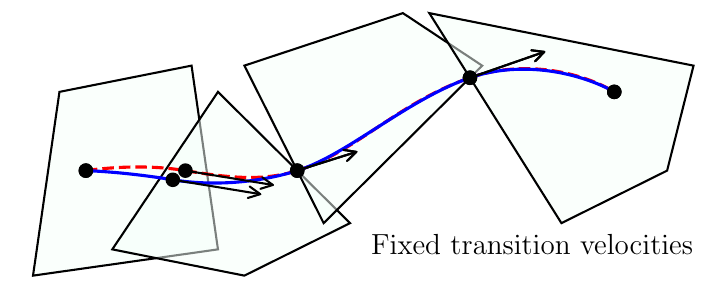} \\ \vspace{-1mm}
\includegraphics[width=\columnwidth]{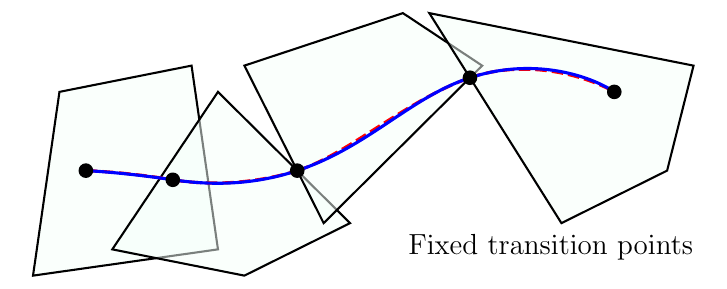} \\
\vspace{-15pt} $\thickvdots$ \\
\includegraphics[width=\columnwidth]{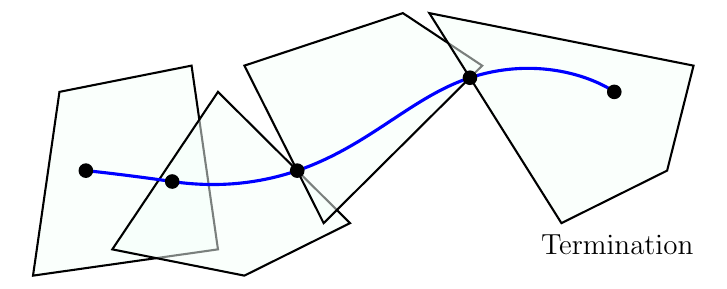}
\caption{Steps of SCS during the solution of the problem in Fig.~\ref{fig:statement}.
In the initialization (1st panel), we quickly compute a feasible polygonal trajectory.
Then, we alternate between a convex subproblem with fixed transition points and one with fixed transition velocities (2nd to 4th panels).
We terminate when the cost decrease of an iteration is small enough (5th panel).
The trajectory computed at each iteration (solid blue) is overlaid on the trajectory from the previous iteration (dashed red).
}
\label{fig:algo}
\end{figure}

The following sections detail our algorithm.
We first illustrate the subproblem with fixed transition velocities, then the one with fixed transition points, and lastly the initialization step.
This order simplifies the exposition, although it is the opposite order of how these steps appear in our algorithm.

\section{Subproblem with Fixed Transition Velocities}
\label{sec:fixed_velocities}

This section illustrates the convex subproblem with fixed transition velocities $\dot \bq(t_1), \ldots, \dot \bq(t_{I-1})$.
First, we formulate problem~\eqref{eq:statement} as a more tractable nonconvex program.
Then, we convexify this program by fixing the transition velocities and approximating the acceleration constraint~\eqref{eq:statement_acc}.

\subsection{Change of variables}

We parameterize the trajectory within each safe set $\cQ_i$ using a function $\bq_i:[0, 1] \rightarrow \reals^n$ and a scalar $T_i > 0$.
These decide the trajectory shape and traversal time, respectively.
We also define the function $h_i(t) = (t - t_{i-1}) / T_i$
%$$
%h_i(t) = \frac{t - t_{i-1}}{T_i},
%$$
that maps the interval of time $[t_{i-1}, t_i]$ assigned to the set $\cQ_i$ to the unit interval $[0,1]$.
This allows us to reconstruct our trajectory as
$$
\bq(t) = \bq_i (h_i(t)),
$$
for all $t \in [t_{i-1}, t_i]$ and $i \irange{1}{I}$.
By differentiating the last equality, we obtain the following expressions for the trajectory velocity and acceleration:
$$
\dot \bq(t) = \frac{\dot \bq_i (h_i(t))}{T_i} , \qquad
\ddot \bq(t) = \frac{\ddot \bq_i (h_i(t))}{T_i^2},
$$
which hold for all $t \in [t_{i-1}, t_i]$ and $i \irange{1}{I}$.

\subsection{Nonconvex formulation}

We express problem~\eqref{eq:statement} in terms of the new variables.
The objective function~\eqref{eq:statement_obj} simply becomes
\begin{align}
\label{eq:pos_form_obj}
\sum_{i=1}^I T_i.
\end{align}
The boundary conditions in~\eqref{eq:statement_init} to~\eqref{eq:statement_boundary_vel} become
\begin{align}
\label{eq:pos_form_bound}
\bq_1(0) = \bq\init, \quad
\bq_I(1) = \bq\term, \quad
\dot \bq_1(0) = \dot \bq_I(1) = \bzero,
\end{align}
where in the last constraint we canceled the traversal times $T_1$ and $T_I$ since the right-hand side is zero.
The next conditions ensure that the trajectory and its derivative are continuous:
\begin{subequations}
\label{eq:pos_form_cont}
\begin{align}
\label{eq:pos_form_cont_pos}
& \bq_i(1) = \bq_{i+1}(0), && i \irange{1}{I-1}, \\
\label{eq:pos_form_cont_vel}
& \frac{\dot \bq_i(1)}{T_i} = \frac{\dot \bq_{i+1}(0)}{T_{i+1}}, && i \irange{1}{I-1}.
\end{align}
\end{subequations}
The constraints in~\eqref{eq:statement_pos},~\eqref{eq:statement_vel}, and~\eqref{eq:statement_acc} become
\begin{subequations}
\label{eq:pos_form_const}
\begin{align}
\label{eq:pos_form_pos}
& \bq_i(s) \in \cQ_i, && s \in [0,1], \ i \irange{1}{I}, \\
\label{eq:pos_form_vel}
& \dot \bq_i(s) \in T_i \cV, && s \in [0,1], \ i \irange{1}{I}, \\
\label{eq:pos_form_acc}
& \ddot \bq_i(s) \in T_i^2 \cA, && s \in [0,1], \ i \irange{1}{I},
\end{align}
\end{subequations}
where we multiplied both sides of the velocity and the acceleration constraints by $T_i$ and $T_i^2$, respectively.
Finally, constraint~\eqref{eq:statement_dur} results in
\begin{align}
\label{eq:pos_form_dur}
& T_i > 0, && i \irange{1}{I},
\end{align}
where zero traversal times are excluded because of Assumption~\ref{ass:zero_time}.

Overall, problem~\eqref{eq:statement} is reformulated as
\begin{equation}
\label{eq:pos_form}
\begin{aligned}
\minimize \quad & \eqref{eq:pos_form_obj} \\
\subjectto \quad &  \text{\eqref{eq:pos_form_bound} to \eqref{eq:pos_form_dur}},
\end{aligned}
\end{equation}
with variables $\bq_i$ and $T_i$ for $i \irange{1}{I}$.
The objective and most of the constraints of this problem are linear.
The position constraint~\eqref{eq:pos_form_pos} is convex.
As mentioned in~\S\ref{sec:notation}, also the velocity constraint~\eqref{eq:pos_form_vel} is convex.
On the other hand, the velocity continuity~\eqref{eq:pos_form_cont_vel} and the acceleration constraint~\eqref{eq:pos_form_acc} are nonconvex.
Therefore, the overall problem is nonconvex.

\subsection{Convex restriction}

The next step is to construct a \emph{convex restriction} (informally, a convex inner approximation~\cite[\S2.1]{diamond2018general}) of the nonconvex constraints of problem~\eqref{eq:pos_form}.
To do so, we assume that the transition velocities have fixed value,
\begin{align*}
& \dot \bq(t_i) = \bv_i,
&& i \irange{1}{I-1},
\end{align*}
and that we are given nominal values $\bar T_i > 0$ for the traversal times $T_i$, for all $i \irange{1}{I}$.

With the transition velocities fixed, the velocity-continuity constraints~\eqref{eq:pos_form_cont_vel} become linear:
\begin{align}
\label{eq:fixed_vel_cont}
& \dot \bq_i(1) = \bv_i T_i, \
\dot \bq_{i+1}(0) = \bv_i T_{i+1},
&& i \irange{1}{I-1}.
\end{align}
To approximate the acceleration constraint~\eqref{eq:pos_form_acc}, we underestimate the convex function $T_i^2$ with its linearization around the nominal value $\bar T_i$:
\begin{align}
\label{eq:acc_lower_bound}
T_i^2 \geq \bar T_i (2 T_i - \bar T_i).
\end{align}
%which holds for all $T_i$  and becomes an equality when $T_i = \bar T_i$.
We then replace~\eqref{eq:pos_form_acc} with
\begin{subequations}
\label{eq:fixed_vel_acc_const}
\begin{align}
\label{eq:fixed_vel_acc_const_perspective}
& \ddot \bq_i(s) \in \bar T_i (2 T_i - \bar T_i)\cA, && s \in [0,1], \ i \irange{1}{I}, \\
\label{eq:fixed_vel_acc_const_min_dur}
& 2 T_i \geq \bar T_i,&& i \irange{1}{I}.
\end{align}
\end{subequations}
These constraints are convex (see again the discussion in~\S\ref{sec:notation}).
Moreover, they imply~\eqref{eq:pos_form_acc} because of the inequality~\eqref{eq:acc_lower_bound} and the assumption that the constraint set $\cA$ contains the origin.

Collecting all the pieces, we have the following convex restriction of problem~\eqref{eq:pos_form}:
\begin{equation}
\label{eq:fixed_vel}
\begin{aligned}
\minimize \quad & \eqref{eq:pos_form_obj} \\
\subjectto \quad
& \text{\eqref{eq:pos_form_bound} to \eqref{eq:pos_form_const} except \eqref{eq:pos_form_cont_vel} and \eqref{eq:pos_form_acc}}, \\
& \text{\eqref{eq:fixed_vel_cont} and \eqref{eq:fixed_vel_acc_const}}.
\end{aligned}
\end{equation}
Constraint~\eqref{eq:pos_form_dur} is omitted here since it is implied by~\eqref{eq:fixed_vel_acc_const_min_dur}.
Given a feasible trajectory with transition velocities $\bv_1, \ldots, \bv_{I-1}$ and traversal times $\bar T_1, \ldots, \bar T_I$, this problem yields another feasible trajectory with lower or equal cost.

\section{Subproblem with Fixed Transition Points}
\label{sec:fixed_points}

We now illustrate the convex subproblem with fixed transition points $\bq(t_1), \ldots, \bq(t_{I-1})$.
In the previous section, we parameterized the trajectory at the ``position level'' using the functions $\bq_i$ for $i\leq I$.
The velocity and acceleration were $\dot \bq_i/T_i$ and $\ddot \bq_i/T_i^2$, respectively.
This choice made all the position constraints convex, and gave us some nonconvex velocity and acceleration constraints.
Here, we parameterize the trajectory at the ``velocity level'' using the functions $\dot \br_i = \dot \bq_i/T_i$.
We recover the position as $T_i \br_i$ and the acceleration as $\ddot \br_i/T_i$.
Furthermore, we work with the reciprocals $S_i = 1 /T_i$ of the traversal times.
This yields a problem equivalent to~\eqref{eq:pos_form} where all the velocity constraints are convex, and the nonconvexities are only at the position and acceleration levels:
\begin{subequations}
\label{eq:vel_form}
\begin{align}
\label{eq:vel_form_obj}
\minimize \quad & \sum_{i=1}^I \frac{1}{S_i} \\
\label{eq:vel_form_bound_pos}
\subjectto \quad
& \br_1(0) = S_1 \bq\init, \\
& \br_I(1) = S_I \bq\term, \\
\label{eq:vel_form_bound_vel}
& \dot \br_1(0) = \dot \br_I (1) = \bzero, \\
\label{eq:vel_form_cont_pos}
& \frac{\br_i(1)}{S_i} = \frac{\br_{i+1}(0)}{S_{i+1}}, && i \irange{1}{I-1}, \\
\label{eq:vel_form_cont_vel}
& \dot \br_i(1) = \dot \br_{i+1}(0), && i \irange{1}{I-1}, \\
\label{eq:vel_form_pos}
& \br_i(s) \in S_i \cQ_i, && s \in [0,1], \ i \irange{1}{I}, \\
\label{eq:vel_form_vel}
& \dot \br_i(s) \in \cV, && s \in [0,1], \ i \irange{1}{I}, \\
\label{eq:vel_form_acc}
& \ddot \br_i(s) \in (1/S_i) \cA, && s \in [0,1], \ i \irange{1}{I}, \\
\label{eq:vel_form_dur}
& S_i > 0, && i \irange{1}{I}.
\end{align}
\end{subequations}
Observe that the objective of this problem is still convex, even though we work with the traversal-time reciprocals.
The only nonconvex constraints are the position continuity~\eqref{eq:vel_form_cont_pos} and the acceleration constraint~\eqref{eq:vel_form_acc}, which have the same structure as the constraints~\eqref{eq:pos_form_cont_vel} and~\eqref{eq:pos_form_acc}, respectively.

We proceed as in the previous section to construct a convex restriction of problem~\eqref{eq:vel_form}.
This time we assume that the transition points are fixed:
\begin{align*}
& \bq(t_i) = \bp_i,
&& i \irange{1}{I-1}.
\end{align*}
This makes the position continuity~\eqref{eq:vel_form_cont_pos} linear:
\begin{align}
\label{eq:fixed_pos_cont}
\br_i(1) = \bp_i S_i, \
\br_{i+1}(0) = \bp_i S_{i+1},
&& i \irange{1}{I-1}.
\end{align}
To approximate the acceleration constraint~\eqref{eq:vel_form_acc}, we assume again  that we are given nominal values $\bar T_i > 0$ of the traversal times.
We underestimate the convex function $1/S_i$ with its linearization around the nominal point:
$$
\frac{1}{S_i} \geq \bar T_i (2 - \bar T_i S_i).
$$
This gives us the following convex restriction of the acceleration constraint:
\begin{subequations}
\label{eq:fixed_pos_acc_const}
\begin{align}
\label{eq:fixed_pos_acc_const_perspective}
& \ddot \br_i(s) \in \bar T_i (2 - \bar T_i S_i)\cA, && s \in [0,1], \ i \irange{1}{I}, \\
\label{eq:fixed_pos_acc_const_nonneg}
& 2 \geq \bar T_i S_i,&& i \irange{1}{I}.
\end{align}
\end{subequations}

Overall, the subproblem with fixed transition points is
\begin{equation}
\label{eq:fixed_pos}
\begin{aligned}
\minimize \quad & \eqref{eq:vel_form_obj} \\
\subjectto \quad
& \text{\eqref{eq:vel_form_bound_pos} to \eqref{eq:vel_form_dur} except \eqref{eq:vel_form_cont_pos} and \eqref{eq:vel_form_acc}}, \\
& \text{\eqref{eq:fixed_pos_cont} and \eqref{eq:fixed_pos_acc_const}}.
%\text{\eqref{eq:vel_form_bound_pos} to \eqref{eq:vel_form_bound_vel},
%\eqref{eq:vel_form_cont_vel} to \eqref{eq:vel_form_vel},
%\eqref{eq:vel_form_dur},
%\eqref{eq:fixed_pos_cont},
%\eqref{eq:fixed_pos_acc_const}.}
\end{aligned}
\end{equation}
This convex subproblem allows us to improve a given feasible trajectory with transition points $\bp_1, \ldots, \bp_{I-1}$ and traversal times $\bar T_1, \ldots, \bar T_I$.

\section{Initialization with Polygonal Trajectory}
\label{sec:initialization}

In the initialization of SCS we quickly identify a low-cost feasible trajectory for problem~\eqref{eq:statement}.
As in the proof of Proposition~\ref{prop:feas}, a natural candidate for this role is a polygonal trajectory that comes to a full stop at each ``kink.''

The shape of our polygonal trajectory is computed through the following convex program:
\begin{subequations}
\label{eq:polygonal}
\begin{align}
\minimize \quad
& \sum_{i=0}^{I-1} \| \bp_{i+1} - \bp_i\|_2\\
\subjectto \quad
& \bp_0 = \bq\init, \\
& \bp_I = \bq\term, \\
& \bp_i \in \cQ_i \cap \cQ_{i+1}, && i \irange{1}{I-1}.
\end{align}
\end{subequations}
Here the decision variables are the points $\bp_0, \ldots, \bp_I$ that the trajectory interpolates through straight lines (black dots in the first panel of Fig.~\ref{fig:algo}).
The objective minimizes the total Euclidean length of the trajectory.

Next, we select the \emph{vertices} of the polygonal trajectory, i.e., the points $\bp_i$ that do not lie on the line connecting $\bp_{i-1}$ to $\bp_{i+1}$.
(This condition can be efficiently checked using the triangle inequality.)
%The triangle inequality gives us an efficient way of checking this: $\bp_i$ is a vertex if and only if
%$$
%\| \bp_i - \bp_{i-1}\|_2 + \| \bp_{i+1} - \bp_i\|_2 > \| \bp_{i+1} - \bp_{i-1}\|_2.
%$$
For ease of notation, we also include $\bp_0$ and $\bp_I$ in the list of vertices.
As an example, in the top panel of Fig.~\ref{fig:algo}, the only point that is not a vertex is $\bp_1$ (the second).

The initialization is completed by connecting each pair of consecutive vertices through a minimum-time trajectory segment, with zero velocity at the endpoints.
While these vertex-to-vertex problems could be solved in closed form when working with infinite-dimensional trajectories, in practice, we use a finite-dimensional trajectory parameterization, and it is convenient to formulate them as convex programs.
To this end, let us assume that we are connecting two vertices that are consecutive points $\bp_{i-1}$ and $\bp_i$.
(If not, we can proceed as follows and, afterwards, split the designed trajectory into pieces.)
The vertex-to-vertex problem can be formulated as a convex program similar to the nonconvex problem~\eqref{eq:vel_form}:
\begin{subequations}
\label{eq:parameterization_cvx}
\begin{align}
\minimize \quad
& T_i \\
\subjectto \quad
& \br_i(0) = S_i \bp_{i-1}, \\
& \br_i(1) = S_i \bp_i, \\
& \dot \br_i(0) = \dot \br_i(1) = \bzero, \\
& \dot \br_i(s) \in \cV,&& s \in [0,1], \\
& \ddot \br_i(s) \in T_i \cA,&& s \in [0,1], \\
& T_i \geq 1 / S_i, \ S_i > 0.
\end{align}
\end{subequations}
The variables are the traversal time $T_i$, its reciprocal $S_i$, and the function $\br_i :[0,1] \rightarrow \reals^n$ (which represents $\bq_i / T_i$).
%The minimum-time trajectory segment connecting the two vertices is recovered as $T_i \br_i$, and has duration $T_i$.
The last constraint relaxes the nonconvex equality $T_i = 1 / S_i$ to a convex inequality.
However, this relaxation is lossless: in fact, given any feasible solution $\bar T_i$, $\bar S_i$, and $\bar \br_i$, the solution $T_i = \bar T_i$, $S_i = 1 / \bar T_i$, and $\br_i = \bar \br_i / (\bar T_i \bar S_i)$ is also feasible, has equal cost, and satisfies $T_i = 1 / S_i$.
In practice, we solve problem~\eqref{eq:parameterization_cvx} as a one-dimensional problem, leveraging the fact that its optimal trajectories are straight lines.
This accelerates our algorithm when working in high-dimensional spaces.

\section{Numerical Implementation}
\label{sec:num}

The numerical implementation of our method requires a finite-dimensional trajectory parameterization.
In some special cases, it is possible to use a parameterization that captures the infinite-dimensional optimum of problem~\eqref{eq:statement}.
%For example, if all the convex sets in problem~\eqref{eq:statement} are axis-aligned boxes, then optimal trajectories have piecewise-constant acceleration.
However, in general, optimal trajectories can be quite complex, and some approximation error is unavoidable.

B\'ezier curves have been widely used in motion planning, and enjoy several properties that make them particularly well suited for our problems.
In this section, we first collect some basic definitions and properties of B\'ezier curves, following~\cite[\S{V}-A]{marcucci2024fast}.
Then we show how the infinite-dimensional problems in the previous sections can be translated into efficient finite-dimensional programs.

\subsection{B\'ezier curves}
\label{sec:bez}

B\'ezier curves are constructed using Bernstein polynomials.
The \emph{Bernstein polynomials} of degree $K$ are defined over the interval $[0,1] \subset \reals$ as follows:
\begin{align}
\label{eq:berst}
\beta_k (s) = \binom{K}{k}
s^k
(1 - s)^{K-k},
\quad k \krange{0}{K}.
\end{align}
(Recall that in this paper $k$ is nonnegative and $k \leq K$ is shorthand for $k \in \{0,\ldots,K\}$.)
The Bernstein polynomials are nonnegative and, by the binomial theorem, sum up to one.
Therefore, the scalars $\beta_0(s), \ldots, \beta_K(s)$ represent the coefficients of a convex combination for all $s \in [0,1]$.
We use these coefficients to combine a given set of \emph{control points} $\bgamma_0, \ldots, \bgamma_K \in \reals^n$, and obtain a \emph{B\'ezier curve}:
\begin{align}
\label{eq:bez}
\bgamma(s) = \sum_{k=0}^K \beta_k (s)  \bgamma_k.
\end{align}
The function $\bgamma:[0,1] \rightarrow \reals^n$ is a (vector-valued) polynomial of degree $K$.
Fig.~\ref{fig:bez} shows a B\'ezier curve of degree $K=4$ in $n=2$ dimensions.

\begin{figure}
\centering
\includegraphics[width=\columnwidth]{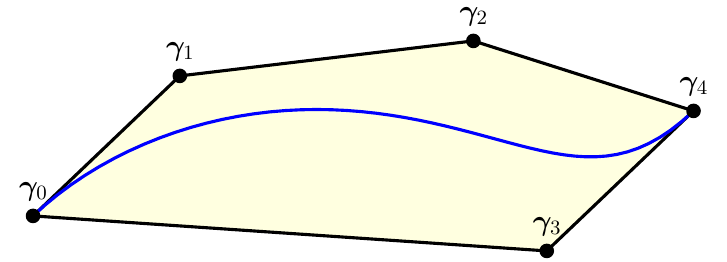}
\caption{A two-dimensional B\'ezier curve with control points $\gamma_0, \ldots, \gamma_4$.
The area shaded in yellow is the convex hull of the control points.}
\label{fig:bez}
\end{figure}

The following are a few selected properties of B\'ezier curves.
We refer to~\cite{farouki1988algorithms} for a more comprehensive list.

%\begin{property}[Polynomial basis]
%\label{prop:basis}
%The Bernstein polynomials $\beta_0, \ldots, \beta_K$ form a basis for the space of polynomials of degree at most $K$.
%%~\cite[\S1.3]{farouki1988algorithms}.
%Therefore, parameterizing a trajectory through the B\'ezier curve~\eqref{eq:bez} is as expressive as using a polynomial function of degree at most $K$.
%\end{property}

\begin{property}[Derivative]
\label{prop:der}
%The derivative $\dot \gamma$ of a B\'ezier curve $\bgamma$ of degree $K$ is a polynomial function of degree $K-1$.
%Therefore, it is representable as a B\'ezier curve of degree $K-1$.
The derivative $\dot \gamma$ of the B\'ezier curve $\bgamma$ is a B\'ezier curve of degree $K-1$.
Its control points are computed via the linear difference equation
$$
\dot \gamma_k = K (\bgamma_{k+1} - \bgamma_k),
\quad k \krange{0}{K-1}.
$$
\end{property}

\begin{property}[Endpoint]
\label{prop:endpoint}
The B\'ezier curve $\bgamma$ starts at its first control point and ends at its last control point:
$$
\bgamma(0) = \bgamma_0, \qquad
\bgamma(1) = \bgamma_K.
$$
\end{property}

\begin{property}[Convex hull]
\label{prop:ch}
The B\'ezier curve $\bgamma$ is contained in the convex hull of its control points at all times:
$$
\bgamma(s) \in \conv (\{\bgamma_0, \ldots, \bgamma_K\}),
\quad s \in [0,1].
$$
This convex hull is shaded in yellow in Fig.~\ref{fig:bez}.
\end{property}

%\begin{property}[Domain split]
%\label{prop:split}
%If we split a B\'ezier curve $\bgamma:[0,1] \rightarrow \reals^n$ of degree $K$ at time $s \in (0,1)$, we obtain two polynomial functions $\bgamma' : [0,s] \rightarrow \reals^n$ and $\bgamma'': [s, 1] \rightarrow \reals^n$ of degree $K$.
%These can be expressed in B\'ezier form using De Casteljau's algorithm.
%First, we let $\bgamma_k^{0} = \bgamma_k$ for all $k \krange{0}{K}$.
%Then, we recursively define 
%\begin{align*}
%& \bgamma_k^{j} = (1 - s) \bgamma_k^{j-1} + s \bgamma_{k+1}^{j-1} ,
%&& k \krange{0}{K-j}, \ j \in \intint{1}{K}.
%\end{align*}
%The control points of $\bgamma'$ and $\bgamma''$ are equal to $\bgamma_0^0, \ldots, \bgamma_0^K$ and $\bgamma_0^K, \ldots, \bgamma_K^0$, respectively.
%\end{property}

%\begin{property}[Line segment]
%\label{prop:line}
%If the control points of a B\'ezier curve lie on a line, then, by the convex-hull Property~\ref{prop:ch}, the whole curve lies on the same line.
%If in addition the first $J$ control points overlap and the last $J$ control points overlap, then the B\'ezier curve connects $\bgamma_0$ and $\bgamma_K$ through a straight-line, with $J$ derivatives equal to zero at the endpoints.
%\end{property}

\subsection{Finite-dimensional trajectory parameterization}

When solving the programs~\eqref{eq:fixed_vel},~\eqref{eq:fixed_pos} and~\eqref{eq:parameterization_cvx} numerically, we restrict our trajectory segments ($\bq_i$ or $\br_i$) to B\'ezier curves of degree $K$, and enforce all the necessary constraints leveraging the properties above.
Property~\ref{prop:der} tells us that the trajectory velocity and acceleration are also piecewise B\'ezier curves, of degree $K-1$ and $K-2$, respectively.
Using Property~\ref{prop:endpoint}, we can then easily enforce any boundary or continuity condition by constraining the first and last control points of our B\'ezier curves.
The containment of a trajectory segment (or its derivatives) in a convex set can be enforced using Property~\ref{prop:ch}: if all the control points of a B\'ezier curve lie in a convex set, then so does the whole curve.
In the initialization step, we might also have to split a trajectory segment, obtained by solving problem~\eqref{eq:parameterization_cvx}, into multiple pieces.
This is easily done by using De Casteljau's algorithm~\cite[\S2.4]{farouki1988algorithms}.

For completeness, in \S\ref{app:finite_dim}, we report  the finite-dimensional versions of the convex programs~\eqref{eq:fixed_vel},~\eqref{eq:fixed_pos}, and~\eqref{eq:parameterization_cvx}.
We also report the finite-dimensional version of the nonconvex program~\eqref{eq:pos_form}, which will serve as a baseline in the experiments below.

\section{Strengths}
\label{sec:strengths}

This section illustrates the main strengths of SCS.

\subsection{Convergence and completeness}

Under our assumptions on the problem data, SCS is guaranteed to converge monotonically.
In fact, the initialization step must succeed, since problems~\eqref{eq:polygonal} and~\eqref{eq:parameterization_cvx} are feasible and admit an optimal solution.
Then, the convex subproblems~\eqref{eq:fixed_vel} and~\eqref{eq:fixed_pos} are guaranteed to produce trajectories that are not worse than the ones they are initialized with.
This makes our algorithm \emph{complete} (guaranteed to find a solution) and \emph{anytime} (returns a feasible solution even if stopped early).

These results extend to the finite-dimensional implementation of SCS from \S\ref{sec:num}, provided that our B\'ezier curves have degree $K \geq 3$.
This minimum degree is sufficient for our trajectory segments to represent straight lines with zero endpoint velocity, and ensures the success of the initialization step.
After that, the biconvex alternation can only improve our finite-dimensional trajectory.
Notably, our piecewise B\'ezier trajectories satisfy the constraints of problem~\eqref{eq:statement} at all continuous times, rather than at a finite set of times, as is common for sampling-based and trajectory-optimization methods.

\subsection{Optimality}

SCS is \emph{heuristic}: it is not guaranteed to find an optimal solution (global or local), or to converge within a fixed distance from one.
However, it typically finds high-quality trajectories in a fraction of the time of state-of-the-art solvers (see the experiments in \S\ref{sec:comparison}).
The trajectory parameterization using B\'ezier curves can also affect the optimality of our trajectories.
In this direction, we remark that a B\'ezier curve is as expressive as any polynomial of equal degree~\cite[\S1.3]{farouki1988algorithms}.
Another source of suboptimality are the conservative convex constraints obtained using Property~\ref{prop:ch}.
However, these constraints get arbitrarily accurate as the degree $K$ increases.
Potentially, we could also use exact containment conditions like sums of squares~\cite{parrilo2003semidefinite,deits2015efficient}, but this would make our programs much more expensive to solve.

\subsection{Computational efficiency}

The runtime of an iteration of SCS is polynomial in all the relevant problem data, and linear in the number $I$ of safe sets and the degree $K$ of the B\'ezier curves.
In fact, the subproblems~\eqref{eq:fixed_vel} and~\eqref{eq:fixed_pos} have banded structure, and are solvable in a time that is linear in $I$ and $K$ (see, e.g.,~\cite{wang2009fast}).
Problems~\eqref{eq:polygonal} and~\eqref{eq:parameterization_cvx} are also banded, and solvable in a time that is linear in $I$ and $K$, respectively.
In addition, the latter problem is solved at most $I$ times.

The overall time complexity of SCS is harder to quantify.
However, in practice, we observed that the number of iterations necessary for convergence is often insensitive to $I$ and $K$ (see the experiments in \S\ref{sec:comparison}).
This leads to overall runtimes that are often linear in $I$ and $K$.

\subsection{Limited parameter tuning}

SCS does not require the tuning of any step-size or trust-region parameter.
The only numerical values set by the user are the degree $K$ of the B\'ezier curves and the convergence tolerance $\varepsilon$.
The first should be at least three to ensure convergence, and can be increased to improve the solution quality.
For the second, we have found that $\varepsilon=0.01$ is sufficiently small for most problems.

\subsection{Advantages over existing methods}
\label{sec:related_methods}

As discussed in \S\ref{sec:intro}, SCS addresses a problem similar to the one in~\cite[\S{V}]{marcucci2024fast}.
Compared to that approach, SCS applies to a narrower set of motion-planning problems, but converges much faster (see experiments in \S\ref{sec:sparrow}).
This is because its subproblems are convex restrictions of the original nonconvex program, and at every iteration we can take a full step towards their optima.

The GCS motion planner from~\cite{marcucci2023motion} requires a convex trajectory parameterization within each safe set.
However, as also seen in this paper, this is very challenging when we optimize both the trajectory shape and timing, and imposes strict limitations on the types of costs and constraints that GCS can handle.
For instance, the method in~\cite{marcucci2023motion} can only enforce coarse approximations of the acceleration constraints~\eqref{eq:statement_acc}.
The recent work~\cite{yang2025new} proposes a semidefinite relaxation for these time-scaling problems, broadening the list of costs and constraints that GCS can accommodate but sacrificing the algorithm completeness.
Overall, GCS and SCS can be viewed as complementary methods, and can be combined in hybrid approaches where GCS provides an approximate solution to the high-level discrete-continuous problem and SCS refines the trajectory within a fixed sequence of safe sets.

Optimization problems similar to the one considered in this paper are also faced by UAV motion planners based on safe flight corridors~\cite{chen2016online,liu2017planning,wu2024optimal}.
However, these planners typically bypass the problem nonconvexity by fixing the corridor traversal times using heuristics, while here we optimize these times explicitly.

The main advantage of SCS over general-purpose methods for trajectory optimization is its reliability and completeness.
Furthermore, SCS can generate high-quality trajectories for complex planning problems within a few milliseconds (see \S\ref{sec:sparrow}).
In contrast, trajectory-optimization methods require a GPU to achieve comparable runtimes~\cite{sundaralingam2023curobo}.
Finally, most common trajectory-optimization methods do not take full advantage of the structure of minimum-time problems.

The minimum-distance problem~\eqref{eq:polygonal}, solved to initialize SCS, is similar to the problem addressed by common sampling-based methods.
This step is straightforward for us since we assume that the free space is represented as a sequence of convex sets.
Contrarily, sampling-based methods rely solely on a collision checker, which makes finding a minimum-distance curve significantly more challenging.
The work~\cite{werner2025superfast} explores a combined approach, where a sampling-based method is used to find a polygonal curve that is later inflated into a sequence of safe sets for SCS to plan through.

Finally, various convex relaxations and reformulations of time-optimal control and trajectory-tracking problems have been proposed over the years (see, e.g.,~\cite{verscheure2009time,lipp2014minimum,leomanni2022time,malyuta2022convex}).
However, none of these methods applies directly to the problem of designing trajectories through sequences of convex sets.

\section{Limitations}
\label{sec:limitations}

Our method has a few worth-noting limitations.
First of all, SCS is restricted to minimum-time problems.
However, a similar approach can be applied to problems with fixed final time and cost function that penalizes the magnitude of the trajectory velocity and acceleration.

SCS requires that the robot free space is described as a sequence of convex sets.
This description can be challenging to compute for high-dimensional problems and cluttered environments.
However, as mentioned in \S\ref{sec:intro}, many practical methods for decomposing complex spaces into convex sets are now available, and also GPU-based algorithms have been recently developed~\cite{werner2025superfast}.

The trajectories generated by SCS may have acceleration jumps, which can make them difficult to track on real hardware.
A simple workaround is to add a smoothing step.
Alternatively, we can ensure that the trajectory acceleration (as well as any higher-order derivative) is continuous by setting it to zero at the transition times.
This is easily seen to be a linear constraint.
A similar limitation is that SCS can only handle constraints on the velocity and acceleration but not, for example, on the trajectory jerk.

We have seen that SCS cannot handle problems where an optimal traversal time $T_i$ is zero (in which case the corresponding variable $S_i$ in the subproblem with fixed transition points~\eqref{eq:fixed_pos} is infinity).
Although Assumption~\ref{ass:zero_time} is sufficient to rule out this scenario, some practically relevant problems do not meet this assumption.
In these cases, we can enforce an artificial lower bound on the time spent in each safe set.

\section{Numerical Experiments}
\label{sec:experiments}

We demonstrate SCS on three numerical experiments.
First, we conclude the simple running example in Fig.~\ref{fig:statement} and~\ref{fig:algo} by reporting its solution statistics.
Second, we analyze the performance of SCS as a function of multiple problem data, and we compare it with state-of-the-art solvers for nonconvex optimization.
Finally, we demonstrate SCS on a minimum-time package-transfer problem with two Sparrow robots, and we benchmark it against other motion-planning methods.

The Python implementation of SCS used in the experiments below is available at
$$
\href{https://github.com/TobiaMarcucci/scsplanning}{\texttt{github.com/TobiaMarcucci/scsplanning}}.
$$
It is based on \texttt{Drake}~\cite{tedrake2019drake}, and uses the open-source solver \texttt{Clarabel}~\cite{goulart2024clarabel} for the convex programs.
All the experiments are run on a laptop with Apple M2 Pro processor and 16 GB of RAM.
The solvers \texttt{SNOPT}~\cite{gill2005snopt} and \texttt{IPOPT}~\cite{wachter2006implementation} are also called through \texttt{Drake}'s Python interface (and are warm started with the same polygonal trajectory as SCS).

\subsection{Running example}

We provide here the details of the running example illustrated in Fig.~\ref{fig:statement}.
% in Fig.~\ref{fig:statement}  and~\ref{fig:algo}.
The initial and terminal points are $\bq\init = (0,0)$ and $\bq\term = (10,1.5)$, respectively.
The geometry of the safe sets can be deduced from the figure.
The constraint sets $\cV$ and $\cA$ are circles centered at the origin of radius $10$ and $1$, respectively.
The trajectory in Fig.~\ref{fig:statement} has time duration $T=7.45$, and is designed by SCS with degree $K = 5$ and termination tolerance $\varepsilon = 0.01$.

The curves in Fig.~\ref{fig:algo} represent the actual iterations of SCS.
The initial polygonal trajectory has time duration $T=12.49$ (1st panel).
This value decreases to $8.82$ in the first subproblem with fixed transition points (2nd panel), then to $8.06$ and $7.51$ in the subsequent subproblems (3rd and 4th panels).
SCS converges after solving only five subproblems.

As a baseline for SCS, we solve the finite-dimensional version of the nonconvex program~\eqref{eq:pos_form} with \texttt{SNOPT} and \texttt{IPOPT}.
This problem is stated in \S\ref{app:finite_dim}, see~\eqref{eq:nonconv_finite}, and uses the same trajectory parameterization as SCS.
Both solvers yield the trajectory duration $T=7.40$, which is only $0.7\%$ shorter than ours.
Our simple Python implementation of SCS takes $10$~ms to converge, while \texttt{SNOPT} takes $21$~ms and \texttt{IPOPT} needs $261$~ms.
Note, however, that these solvers use smaller termination tolerances than SCS.
Increasing the optimality tolerances of the nonconvex solvers does not reduce their runtimes significantly.
Conversely, if we decrease the SCS tolerance to, e.g., $\varepsilon = 10^{-4}$, the objective gap between SCS and the nonconvex solvers decreases to $0.1\%$, but the runtime of SCS increases to $50$~ms.
This is typical for multi-convex methods: they find high-quality solutions quickly, but can be slow if we seek very accurate solutions~\cite{shen2017disciplined}.

\subsection{Runtime analysis and comparison with nonconvex solvers}
\label{sec:comparison}

We analyze the runtimes of SCS, \texttt{SNOPT}, and \texttt{IPOPT} as functions of several problem parameters: the number $I$ of safe sets, the number $m$ of facets of each safe set, the space dimension $n$, and the trajectory degree $K$.
We show that, across a wide range of problem instances, SCS finds low-cost trajectories more quickly and reliably than the two state-of-the-art solvers.

We construct an instance of problem~\eqref{eq:statement} where each safe set $\cQ_i$ represents one link of an $n$-dimensional staircase.
The safe sets are polytopes that approximate ellipsoids with increasing accuracy as their number $m$ of facets grows.
Fig.~\ref{fig:nom} shows an instance of this problem with the corresponding optimal trajectory.
In this instance, we have $I = 5$ safe sets in $n = 2$ dimensions, and each set has $m=4$ facets (rectangular safe sets).
More details on the construction of these problems are reported in \S\ref{app:staircase}.

\begin{figure}
\centering
\includegraphics[width=.65\columnwidth]{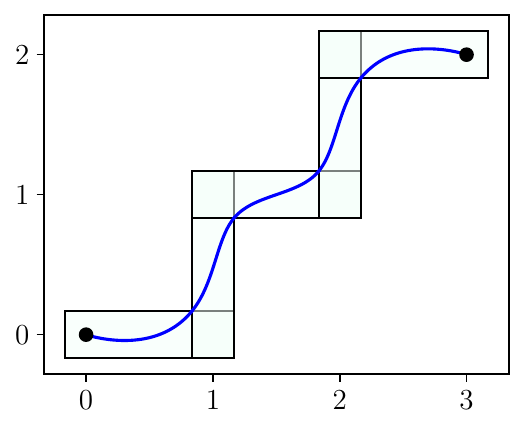}
\caption{Benchmark problem with $I=5$ safe sets in $n=2$ dimensions, with $m=4$ facets each. The optimal trajectory is shown in blue.}
\label{fig:nom}
\end{figure}

We consider a first batch of instances where we let the number $I$ of safe sets grow from $3$ to $3000$, while we fix the space dimension to $n=3$, the number of facets to $m=6$, and  the trajectory degree to $K=3$.
The top panel of  Fig.~\ref{fig:comparison} shows the runtimes of SCS, \texttt{SNOPT}, and \texttt{IPOPT}.
The two nonconvex solvers return trajectories with equal cost, when \texttt{SNOPT} does not fail or reach our time limit of $1$~h (missing markers in the figure).
SCS designs trajectories that have slightly higher cost ($1.2\%$ in the worst case).
SCS is faster in almost all instances: \texttt{SNOPT} and \texttt{IPOPT} have comparable runtimes only on the smallest and largest problems, respectively.
The runtimes of SCS increase a little more than linearly: as the number of safe sets grows by a factor of $1000$, its runtimes increase by $3060$.
The number of subproblems necessary for SCS to converge with tolerance $\varepsilon = 0.01$ ranges between $5$ and $8$.

\begin{figure}[t!]
\centering
\includegraphics[width=\columnwidth]{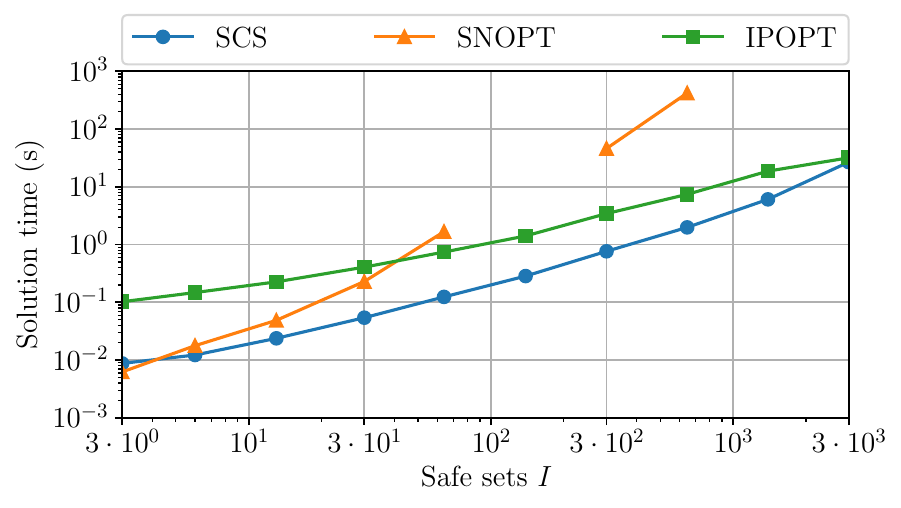} \\ \vspace{-1mm}
\includegraphics[width=\columnwidth]{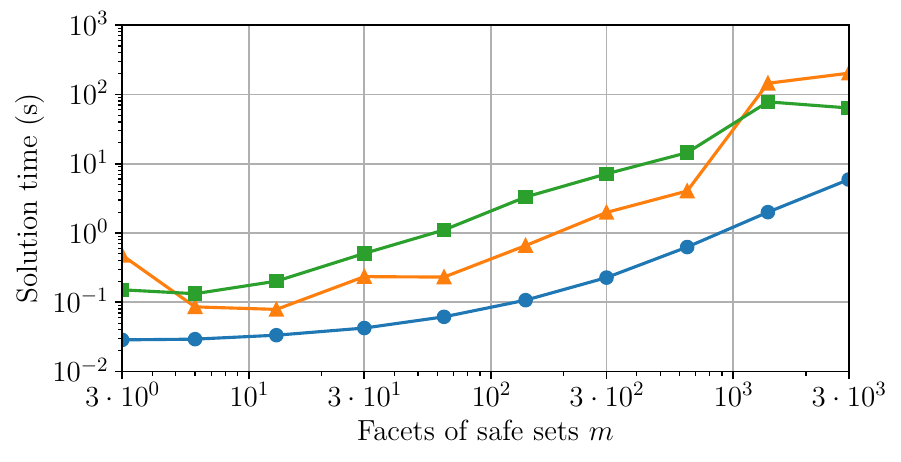} \\ \vspace{-1mm}
\includegraphics[width=\columnwidth]{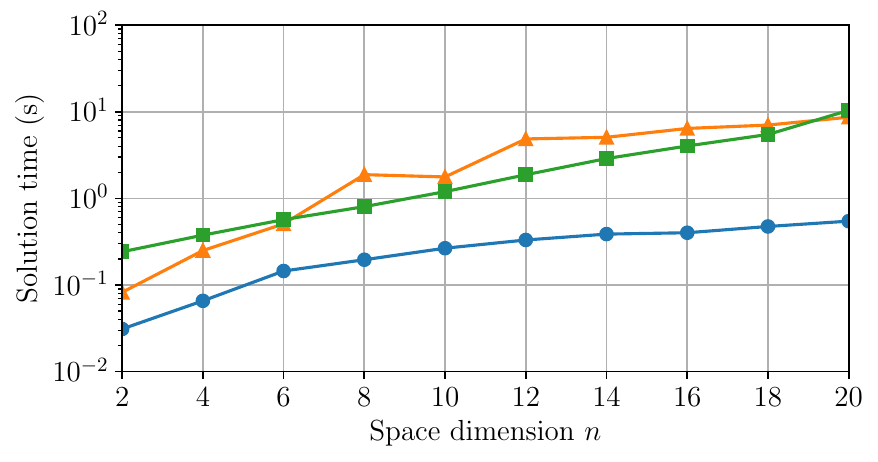} \\ \vspace{-1mm}
\includegraphics[width=\columnwidth]{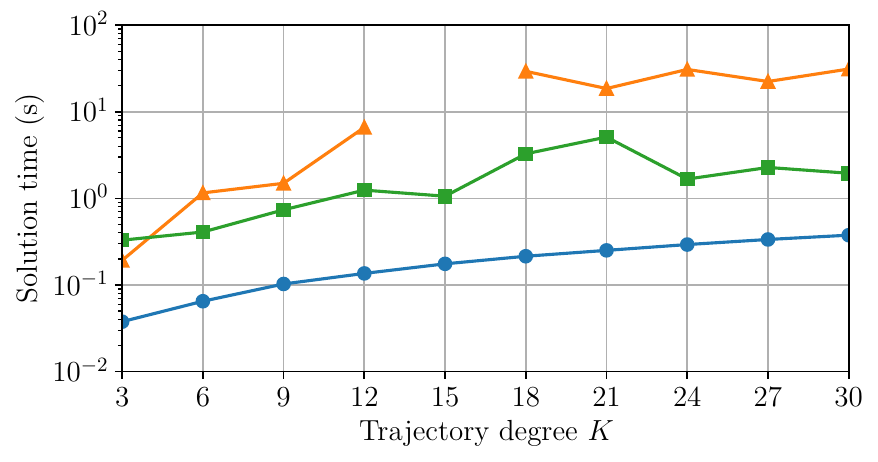}
\caption{Comparison of SCS with the solvers \texttt{SNOPT} and \texttt{IPOPT}.
The runtimes of the three methods are analyzed as functions of multiple problem data.
Missing markers correspond to solver failures.
The runtimes of SCS grow almost linearly in each experiment (note that the horizontal axis has logarithmic scale in the first two panels and linear scale in the last two).
}
\label{fig:comparison}
\end{figure}

The second panel in Fig.~\ref{fig:comparison} shows the effects of increasing the number $m$ of facets of the safe sets from $3$ to $3000$, while keeping $I=20$, $n=2$, and $K=5$.
In this case, SCS and the nonconvex solvers find identical trajectories (despite the larger termination tolerance of SCS).
SCS solves each problem much faster than \texttt{SNOPT} and \texttt{IPOPT}, and its runtimes increase sublinearly with $m$ (as the number of facets grows by $1000$, the runtime grows by $210$).
The number of subproblems necessary for SCS to converge is equal to $5$ for every value of $m$.

In the third panel of Fig.~\ref{fig:comparison}, we let the space dimension $n$ grow from $2$ to $20$, while we set $I=20$, $m=2n$, and $K=3$.
The nonconvex solvers find again identical trajectories, and SCS has a maximum cost gap of $3.2 \%$.
SCS is again the fastest, and its runtimes increase a little more than linearly with $n$ (the space dimension grows by $10$ and the runtimes by $17.6$).
%\texttt{SNOPT} and \texttt{IPOPT} are significantly slower.
The number of SCS subproblems ranges between $5$ and $16$.

In the fourth panel of Fig.~\ref{fig:comparison}, we let $I=20$, $m=6$, $n=3$, and increase the degree $K$ from $3$ to $30$.
All the methods return similar trajectories: the maximum cost difference between SCS and the nonconvex solvers is $0.4\%$.
SCS is the fastest and its runtimes grow linearly with $K$ (the degree increases by $10$ and the runtimes by $9.9$).
\texttt{IPOPT} performs better than \texttt{SNOPT}, which also fails in one instance.
SCS always converges after $5$ subproblems.

\subsection{Minimum-time package transfer with two Sparrow robots}
\label{sec:sparrow}

We use SCS to plan the motion of two Sparrow robots that transfer packages between bins in simulation.
We also benchmark SCS against the trust-region method proposed in~\cite[\S{V}]{marcucci2024fast}, as well as a simple waypoint-based motion planner representative of those commonly used in industry.

The package-transfer task is illustrated in Fig.~\ref{fig:experiment_setup}.
The two robots face each other, and between them is a table with two bins.
One bin contains ten packages and the other is empty.
The goal is to move all the packages in the first bin to the second as quickly as possible.
The final package positions in the second bin must mirror the initial positions in the first bin.
Packages are represented as axis-aligned boxes (these can be the packages themselves, or bounding boxes of products with more complex shape).
The bins have side $0.6$ and height $0.3$, and the distance between their centers is $1$.
The package sides are drawn uniformly at random between $0.1$ and $0.25$.
Also the initial package positions are drawn uniformly at random within the corresponding bin, and sampled packages are rejected when they collide with existing packages.

\begin{figure}
\centering
\includegraphics[width=.7\columnwidth]{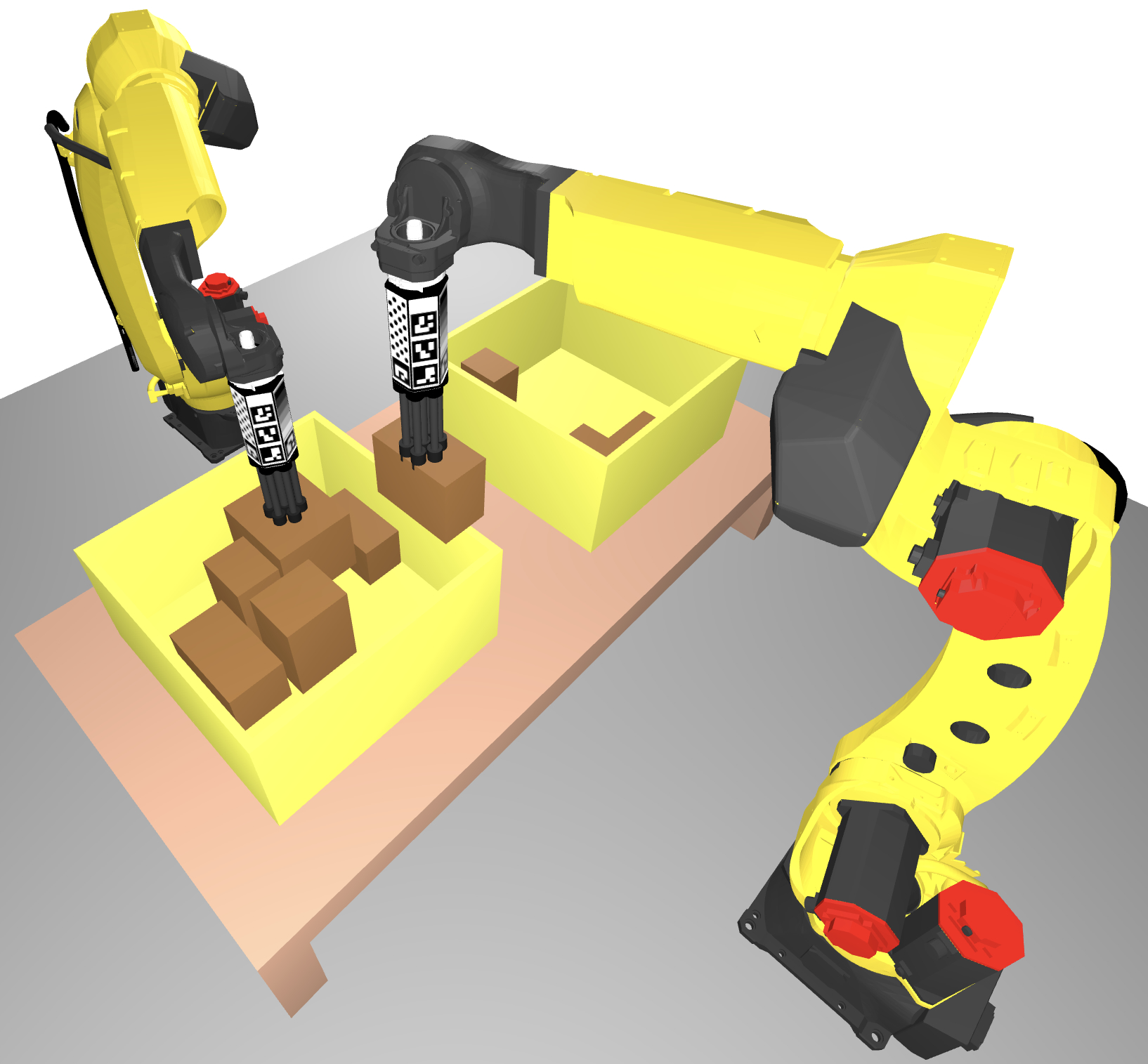}
\caption{Sparrow robots that move packages between bins in minimum time.}
\label{fig:experiment_setup}
\end{figure}

\begin{figure}
\centering
\includegraphics[width=\columnwidth]{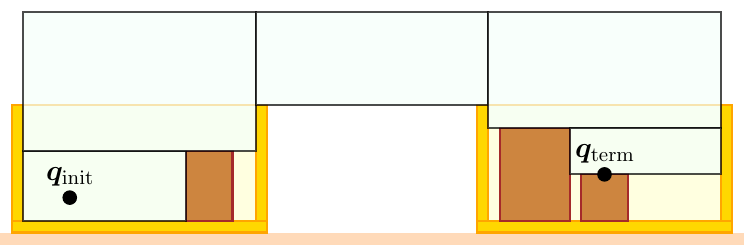} \\
\includegraphics[width=\columnwidth]{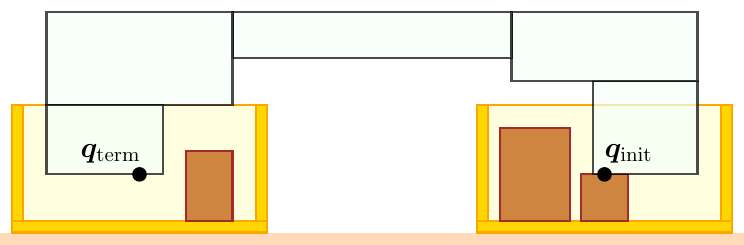}
\caption{Two-dimensional illustration of the three-dimensional safe sets $\cQ_i$ used for the package-transfer task.
The top panel shows the sets for picking the rightmost package.
The bottom panel shows the sets for its placement.
The latter are shrunk to avoid the collision of the transported package.}
\label{fig:robot_regions}
\end{figure}

We solve the task using a state machine.
At each iteration, if a robot has completed its previous pick or place motion, we plan its next motion neglecting the presence of the other robot.
If this results in a collision, we let the robot idle until the next iteration.
If the state machine stalls (neither arm can execute its motion without colliding with the other), we retract one arm and allow the other to move.
Each time a robot plans a picking motion, it targets the package closest to its side of the table.
Trajectories are planned directly in the three-dimensional task space, and the full robot configuration is retrieved through inverse kinematics.
We let the sets $\cV$ and $\cA$, that constrain the gripper velocity and acceleration, be spheres of radius $10$ centered at the origin.
(In practice, these sets can be shaped to prevent package delamination, and ensure that the robots can track the designed task-space trajectories.)
We use B\'ezier curves of degree $K=5$ and set the termination tolerance to $\varepsilon = 0.01$.

For each pick and place motion, the three-dimensional task space is decomposed into five box-shaped safe sets $\cQ_i$, illustrated in two dimensions in Fig.~\ref{fig:robot_regions}.
The first and fifth sets allow the gripper to reach the trajectory endpoints, without colliding with the packages in the bins.
The second and fourth sets cover the space above the packages in the two bins.
The third is a transfer region that connects the spaces above the bins.
As shown in the bottom panel of Fig.~\ref{fig:robot_regions}, these sets are shrunk during a place motion to avoid collisions of the transported package (packages are always picked above their centers).

We consider $50$ randomly generated package-transfer problems.
As the low-level motion planner for the state machine just described, we compare SCS against the following alternatives:
\begin{itemize}
\item 
The trust-region method from~\cite[\S{V}]{marcucci2024fast}, modified as described in \S\ref{app:fpp} to deal with minimum-time problems.
\item 
A simple waypoint-based motion planner, which lifts a package vertically, moves it horizontally above the desired destination, and places it down.
Where each trajectory segment is executed in minimum time.
\end{itemize}
The three methods use the same constraints and trajectory parameterization.
The first two share also the same initialization strategy and termination tolerance.
Tab.~\ref{tab:comparison} shows the statistics for the task-completion time and the runtime of each motion planner.
SCS generates the best trajectories: in fact, the average completion time for the overall package-transfer task is about $10$~s for SCS, $13$~s for the trust-region method, and $15$~s for the waypoint-based planner.
In other words, SCS allows us to transfer $28\%$ and $50\%$ more packages per unit of time than the trust-region and the waypoint-based planners, respectively.
The runtimes of SCS are approximately five times longer than those of the waypoint-based planner, but they remain very low for practical use.
The trust region method is roughly three times slower than SCS.
The videos of five of these package-transfer tasks are provided as Supplementary Material.

\begin{table}
\centering
\caption{Package-Transfer Benchmark}
\begin{tabular}{|c|c|c|c|c|c|c|}
\cline{2-7}
\multicolumn{1}{c|}{} &
\multicolumn{3}{|c|}{Task-completion time (s)}
& \multicolumn{3}{|c|}{Motion-planning time (ms)} \\
\hline
Planner & min & mean & max & min & mean & max \\
\hline \hline
SCS   & 8.86 & 9.96 & 11.34 & 4  & 15   & 49   \\
\hline
Trust region   & 10.43  & 12.73 & 14.87  & 15   & 46   & 163   \\
\hline
Waypoint & 13.17   & 14.97 & 16.74 & 2 & 3 & 12 \\
\hline
\end{tabular}
\label{tab:comparison}
\end{table}

We conclude by emphasizing that the trust-region and the waypoint-based planners are natural baselines for the task considered in this section.
The first provides the same completeness guarantees as SCS, designs smooth trajectories, and has relatively low runtimes.
The second is widespread in warehouse automation thanks to its good performance and high reliability.
%Conversely, more advanced motion planners can struggle significantly with our task.
In our experience, off-the-shelf nonconvex trajectory optimization faces significant challenges with this package-transfer task: it struggles with the many collision geometries in Fig.~\ref{fig:experiment_setup}, relies on handcrafted warm starts, can take seconds to converge (unless we use accelerated hardware~\cite{sundaralingam2023curobo}), and can also fail to converge.
Sampling-based planners can be more reliable, but generate polygonal curves that require additional smoothing.
They excel in tasks where finding a collision-free trajectory is the main challenge, and trajectory cost is secondary.
However, our package-transfer task presents the opposite challenge.

\bibliographystyle{plainnat}
\bibliography{references}

\appendices

\section{Finite-Dimensional Programs}
\label{app:finite_dim}

We illustrate the finite-dimensional versions of the programs presented in this paper.
Here, candidate trajectories are parameterized using B\'ezier curves as shown in \S\ref{sec:num}.

We start from the nonconvex program~\eqref{eq:pos_form}.
The constraints of its finite-dimensional counterpart are as follows (the objective is unchanged):
\begin{subequations}
\label{eq:nonconv_finite}
\begin{align}
& \bq_{1,0} = \bq\init, \\
& \bq_{I,K} = \bq\term, \\
& \dot \bq_{1,0} = \dot \bq_{I, K-1} = \bzero, \\
& \bq_{i,k} \in \cQ_i, && k \krange{0}{K}, \ i \irange{1}{I}, \\
& \dot \bq_{i,k} \in T_i \cV, && k \krange{0}{K-1}, \ i \irange{1}{I}, \\
\label{eq:nonconv_finite_acc}
& \ddot \bq_{i,k} \in T_i^2\cA, && k \krange{0}{K-2}, \ i \irange{1}{I}, \\
\label{eq:nonconv_finite_nonneg}
& T_i > 0, && i \irange{1}{I}, \\
& \bq_{i,K} = \bq_{i+1,0}, && i \irange{1}{I-1}, \\
\label{eq:nonconv_finite_cont_vel}
&  \dot \bq_{i,K-1} / T_i = \dot \bq_{i+1,0} / T_{i+1}, && i \irange{1}{I-1}, \\
& \dot \bq_{i,k} =K (\bq_{i,k+1} - \bq_{i,k}), && k \krange{0}{K-1}, \ i \irange{1}{I}, \\
& \ddot \bq_{i,k} = (K-1) (\dot \bq_{i,k+1} - \dot \bq_{i,k}), && k \krange{0}{K-2}, \ i \irange{1}{I}.
\end{align}
\end{subequations}
For all $i \leq I$, the variables here are the traversal times $T_i$ and the control points $\bq_{i,k}$ for $k \leq K$, $\dot \bq_{i,k}$ for $k \leq K-1$, and $\ddot  \bq_{i,k}$ for $k \leq K-2$.

The finite-dimensional version of the subproblem with fixed transition velocities~\eqref{eq:fixed_vel} differs from the nonconvex program~\eqref{eq:nonconv_finite} in just two ways.
First, the acceleration constraint~\eqref{eq:nonconv_finite_acc} and the traversal-time lower bound~\eqref{eq:nonconv_finite_nonneg} are replaced with the convex conditions
\begin{align*}
& \ddot \bq_{i,k} \in \bar T_i (2 T_i - \bar T_i)\cA, && k \krange{0}{K-2}, \ i \irange{1}{I}, \\
& 2 T_i \geq \bar T_i,&& i \irange{1}{I}.
\end{align*}
Second, the continuity constraint~\eqref{eq:nonconv_finite_cont_vel} is split into two linear constraints:
\begin{align*}
& \dot \bq_{i,K-1} = \bv_i T_i, \ \dot \bq_{i+1,0} = \bv_i T_{i+1}, && i \irange{1}{I-1}.
\end{align*}

The finite-dimensional version of the subproblem with fixed transition points~\eqref{eq:fixed_pos} has similar constraints:
\begin{align*}
& \br_{1,0} = S_1 \bq\init, \\
& \br_{I,K} = S_I \bq\term, \\
& \dot \br_{1,0} = \dot \br_{I, K-1} = \bzero, \\
& \br_{i,k} \in S_i \cQ_i, && k \krange{0}{K}, \ i \irange{1}{I}, \\
& \dot \br_{i,k} \in \cV, && k \krange{0}{K-1}, \ i \irange{1}{I}, \\
& \ddot \br_{i,k} \in \bar T_i (2 - \bar T_i S_i)\cA, && k \krange{0}{K-2}, \ i \irange{1}{I}, \\
& 0 < S_i \leq 2 / \bar T_i,&& i \irange{1}{I}, \\
& \br_{i,K} = \bp_i S_i, \ \br_{i+1,0} = \bp_i S_{i+1}, && i \irange{1}{I-1}, \\
& \dot \br_{i,K-1} = \dot \br_{i+1,0}, && i \irange{1}{I-1}, \\
& \dot \br_{i,k} = K (\br_{i,k+1} - \br_{i,k}), && k \krange{0}{K-1}, \ i \irange{1}{I}, \\
& \ddot \br_{i,k} = (K-1) (\dot \br_{i,k+1} - \dot \br_{i,k}), && k \krange{0}{K-2}, \ i \irange{1}{I}.
\end{align*}
For $i \leq I$, here the decision variables are the  traversal times $T_i$ and the control points $\br_{i,k}$ for $k \leq K$, $\dot \br_{i,k}$ for $k \leq K-1$, and $\ddot  \br_{i,k}$ for $k \leq K-2$.

The initialization phase requires solving problem~\eqref{eq:parameterization_cvx} repeatedly.
Using the B\'ezier parameterization, the constraints of this problem become
\begin{align*}
& \br_{i,0} = S_i \bp_{i-1}, \\
& \br_{i,K} = S_i \bp_i, \\
& \dot \br_{i,0} = \dot \br_{i,K-1} = \bzero, \\
& \dot \br_{i,k} \in \cV,&& k \krange{0}{K-1}, \\
& \ddot \br_{i,k} \in T_i \cA,&& k \krange{0}{K-2}, \\
& T_i \geq 1 / S_i, \ S_i > 0, \\
& \dot \br_{i,k} =K (\br_{i,k+1} - \br_{i,k}), && k \krange{0}{K-1}, \\
& \ddot \br_{i,k} = (K-1) (\dot \br_{i,k+1} - \dot \br_{i,k}), && k \krange{0}{K-2},
\end{align*}
with variables $T_i$, $S_i$, $\br_{i,k}$ for $k \leq K$, $\dot \br_{i,k}$ for $k \leq K-1$, and $\ddot  \br_{i,k}$ for $k \leq K-2$.

\section{Parametric Problem for Runtime Analysis}
\label{app:staircase}

We detail the construction of the problem in \S\ref{sec:comparison}, which is parametric in the number $I$ of safe sets, the number $m$ of facets of each safe set, and the space dimension $n$.

We let $\bx_0 = \bzero \in \reals^n$.
Then, for all $i \irange{1}{I}$, we define the points
$$
\bx_i = \bx_{i-1} + \be_j,
$$
where $\be_j$ is the $j$th element of the standard basis, and $j$ is the reminder of $i/n$.
In other words, the point $\bx_i$ is obtained by shifting $\bx_{i-1}$ along the $j$th dimension by one.
The sequence of points $\bx_0, \ldots, \bx_I$ is then an $n$-dimensional staircase with $I$ links.

For $i \irange{1}{I}$, we let $\cE_i \subset \reals^n$ be an ellipsoid centered around the line segment that connects $\bx_{i-1}$ and $\bx_i$.
The main axis of $\cE_i$ is aligned with the vector $\bx_i - \bx_{i-1}$ and has length $4/3$, while all the other axes have length $1/3$.
Each safe set $\cQ_i$ is a conservative polytopic approximation with $m$ facets of the corresponding ellipsoid $\cE_i$.
In $n=2$ dimensions, this construction over-approximates the unit circle with a regular polytope, which is then mapped to the polytope $\cQ_i$ through the same affine transformation that maps the unit circle to the ellipse $\cE_i$.
We let the initial point be $\bq\init = \bx_0$ and the terminal point be $\bq\term = \bx_I$.
The constraint sets $\cV$ and $\cA$ are spheres centered at the origin of radius $10$ and $1$, respectively.

\section{Trust-Region Method for Minimum-Time Problems}
\label{app:fpp}

We briefly describe how the trust-region method from~\cite[\S{V}]{marcucci2024fast} can be adapted for minimum-time problems.

We start from problem~\eqref{eq:pos_form}, whose only nonconvex constraints are the velocity continuity~\eqref{eq:pos_form_cont_vel} and the acceleration constraint~\eqref{eq:pos_form_acc}.
As in the problem with fixed transition points~\eqref{eq:fixed_pos}, we introduce the variables
\begin{align}
\label{eq:bilinear_constraints}
& \dot \br_i = \dot \bq_i / T_i, && i \irange{1}{I}.
\end{align}
Recall that the derivatives $\ddot \br_i$ of these functions correspond to $\ddot \bq_i/T_i$.
Using these additional variables, the velocity continuity~\eqref{eq:pos_form_cont_vel} becomes a linear constraint
\begin{align*}
& \dot \br_i(1) = \dot \br_{i+1}(0), && i \irange{1}{I-1},
\end{align*}
while the acceleration constraint~\eqref{eq:pos_form_acc} becomes a convex constraint of the kind discussed in \S\ref{sec:notation}:
\begin{align*}
& \ddot \br_i(s) \in T_i \cA, && s \in [0,1], \ i \irange{1}{I}.
\end{align*}
This yields a program whose nonconvexity is due exclusively to the equality constraint~\eqref{eq:bilinear_constraints}.

Following~\cite[\S~V-C]{marcucci2024fast}, we solve the nonconvex program by alternating between two convex programs: a ``tangent'' and ``projection'' program.
In the tangent program, we linearize constraint~\eqref{eq:bilinear_constraints} around the current solution, and try to improve the trajectory shape and timing jointly.
The linearization error is controlled by a trust-region constraint whose size and shape are defined as in~\cite[\S~V-C]{marcucci2024fast}.
Because of the linearization error, the solution of the tangent program might not be feasible for the original nonconvex program.
Therefore, in the projection step, we fix the new trajectory timing and solve the resulting convex program, hoping to obtain a feasible solution with cost lower than the current one.
The numerical solution of these infinite-dimensional convex subproblems follows the steps in \S\ref{sec:num}.

\end{document}

%% file: defs.tex
% sets of numbers
\newcommand{\reals}{\mathbb R}

% subscripts
\newcommand{\init}{_\mathrm{init}}
\newcommand{\term}{_\mathrm{term}}

% operators
\newcommand{\minimize}{\text{minimize}}
\newcommand{\subjectto}{\text{subject to}}
\newcommand{\conv}{\mathrm{conv}}
\newcommand{\interior}{\mathrm{int}}

% sets
\newcommand{\cA}{\mathcal A}
\newcommand{\cC}{\mathcal C}
\newcommand{\cE}{\mathcal E}
\newcommand{\cK}{\mathcal K}
\newcommand{\cQ}{\mathcal Q}
\newcommand{\cS}{\mathcal S}
\newcommand{\cT}{\mathcal T}
\newcommand{\cV}{\mathcal V}

% vectors
\newcommand{\bzero}{\bm 0}
\newcommand{\bone}{\bm 1}
\newcommand{\bgamma}{\bm \gamma}
\newcommand{\bb}{\bm b}
\newcommand{\be}{\bm e}
\newcommand{\bp}{\bm p}
\newcommand{\bq}{\bm q}
\newcommand{\br}{\bm r}
\newcommand{\bv}{\bm v}
\newcommand{\bx}{\bm x}
\newcommand{\bA}{\bm A}

% theorems
\newtheorem{lemma}{Lemma}
\newtheorem{proposition}{Proposition}
\theoremstyle{definition}
\newtheorem{assumption}{Assumption}
\newtheorem{property}{Property}

% miscellaneous
\newcommand{\thickvdots}{\raisebox{0em}{\scalebox{2}{$\vdots$}}}
\newcommand{\irange}[2]{\leq #2}
\newcommand{\krange}[2]{\leq #2}

%% file: main.bbl
\begin{thebibliography}{49}
\providecommand{\natexlab}[1]{#1}
\providecommand{\url}[1]{\texttt{#1}}
\expandafter\ifx\csname urlstyle\endcsname\relax
  \providecommand{\doi}[1]{doi: #1}\else
  \providecommand{\doi}{doi: \begingroup \urlstyle{rm}\Url}\fi

\bibitem[Betts(2010)]{betts2010practical}
John~T Betts.
\newblock \emph{Practical methods for optimal control and estimation using
  nonlinear programming}.
\newblock SIAM, 2010.

\bibitem[Bialkowski et~al.(2011)Bialkowski, Karaman, and
  Frazzoli]{bialkowski2011massively}
Joshua Bialkowski, Sertac Karaman, and Emilio Frazzoli.
\newblock Massively parallelizing the {RRT} and the {RRT}*.
\newblock In \emph{IEEE/RSJ International Conference on Intelligent Robots and
  Systems}, pages 3513--3518. IEEE, 2011.

\bibitem[Boyd and Vandenberghe(2004)]{boyd2004convex}
Stephen Boyd and Lieven Vandenberghe.
\newblock \emph{Convex optimization}.
\newblock Cambridge University Press, 2004.

\bibitem[Chen et~al.(2016)Chen, Liu, and Shen]{chen2016online}
Jing Chen, Tianbo Liu, and Shaojie Shen.
\newblock Online generation of collision-free trajectories for quadrotor flight
  in unknown cluttered environments.
\newblock In \emph{IEEE International Conference on Robotics and Automation},
  pages 1476--1483. IEEE, 2016.

\bibitem[Chia et~al.(2024)Chia, Jiang, Graesdal, Kaelbling, and
  Tedrake]{chia2024gcs}
Shao Yuan~Chew Chia, Rebecca~H Jiang, Bernhard~Paus Graesdal, Leslie~Pack
  Kaelbling, and Russ Tedrake.
\newblock {GCS}*: Forward heuristic search on implicit graphs of convex sets.
\newblock \emph{arXiv preprint arXiv:2407.08848}, 2024.

\bibitem[Dai et~al.(2024)Dai, Amice, Werner, Zhang, and
  Tedrake]{dai2024certified}
Hongkai Dai, Alexandre Amice, Peter Werner, Annan Zhang, and Russ Tedrake.
\newblock Certified polyhedral decompositions of collision-free configuration
  space.
\newblock \emph{The International Journal of Robotics Research}, 43\penalty0
  (9):\penalty0 1322--1341, 2024.

\bibitem[Deits and Tedrake(2015{\natexlab{a}})]{deits2015computing}
Robin Deits and Russ Tedrake.
\newblock Computing large convex regions of obstacle-free space through
  semidefinite programming.
\newblock In \emph{Algorithmic Foundations of Robotics XI: Selected
  Contributions of the Eleventh International Workshop on the Algorithmic
  Foundations of Robotics}, pages 109--124. Springer, 2015{\natexlab{a}}.

\bibitem[Deits and Tedrake(2015{\natexlab{b}})]{deits2015efficient}
Robin Deits and Russ Tedrake.
\newblock Efficient mixed-integer planning for {UAV}s in cluttered
  environments.
\newblock In \emph{IEEE International Conference on Robotics and Automation},
  pages 42--49. IEEE, 2015{\natexlab{b}}.

\bibitem[Diamond et~al.(2018)Diamond, Takapoui, and Boyd]{diamond2018general}
Steven Diamond, Reza Takapoui, and Stephen Boyd.
\newblock A general system for heuristic minimization of convex functions over
  non-convex sets.
\newblock \emph{Optimization Methods and Software}, 33\penalty0 (1):\penalty0
  165--193, 2018.

\bibitem[Farouki and Rajan(1988)]{farouki1988algorithms}
Rida Farouki and V~Rajan.
\newblock Algorithms for polynomials in {B}ernstein form.
\newblock \emph{Computer Aided Geometric Design}, 5\penalty0 (1):\penalty0
  1--26, 1988.

\bibitem[Gill et~al.(2005)Gill, Murray, and Saunders]{gill2005snopt}
Philip~E Gill, Walter Murray, and Michael~A Saunders.
\newblock {SNOPT}: An {SQP} algorithm for large-scale constrained optimization.
\newblock \emph{SIAM Review}, 47\penalty0 (1):\penalty0 99--131, 2005.

\bibitem[Goulart and Chen(2024)]{goulart2024clarabel}
Paul~J Goulart and Yuwen Chen.
\newblock Clarabel: An interior-point solver for conic programs with quadratic
  objectives.
\newblock \emph{arXiv preprint arXiv:2405.12762}, 2024.

\bibitem[Howell et~al.(2019)Howell, Jackson, and Manchester]{howell2019altro}
Taylor~A Howell, Brian~E Jackson, and Zachary Manchester.
\newblock {ALTRO}: A fast solver for constrained trajectory optimization.
\newblock In \emph{IEEE/RSJ International Conference on Intelligent Robots and
  Systems (IROS)}, pages 7674--7679. IEEE, 2019.

\bibitem[Ichnowski et~al.(2020)Ichnowski, Danielczuk, Xu, Satish, and
  Goldberg]{ichnowski2020gomp}
Jeffrey Ichnowski, Michael Danielczuk, Jingyi Xu, Vishal Satish, and Ken
  Goldberg.
\newblock Gomp: Grasp-optimized motion planning for bin picking.
\newblock In \emph{IEEE International Conference on Robotics and Automation},
  pages 5270--5277. IEEE, 2020.

\bibitem[Kalakrishnan et~al.(2011)Kalakrishnan, Chitta, Theodorou, Pastor, and
  Schaal]{kalakrishnan2011stomp}
Mrinal Kalakrishnan, Sachin Chitta, Evangelos Theodorou, Peter Pastor, and
  Stefan Schaal.
\newblock {STOMP}: Stochastic trajectory optimization for motion planning.
\newblock In \emph{IEEE International Conference on Robotics and Automation},
  pages 4569--4574. IEEE, 2011.

\bibitem[Karaman and Frazzoli(2010)]{karaman2010optimal}
Sertac Karaman and Emilio Frazzoli.
\newblock Optimal kinodynamic motion planning using incremental sampling-based
  methods.
\newblock In \emph{49th IEEE Conference on Decision and Control}, pages
  7681--7687. IEEE, 2010.

\bibitem[Karaman and Frazzoli(2011)]{karaman2011sampling}
Sertac Karaman and Emilio Frazzoli.
\newblock Sampling-based algorithms for optimal motion planning.
\newblock \emph{The International Journal of Robotics Research}, 30\penalty0
  (7):\penalty0 846--894, 2011.

\bibitem[Kavraki et~al.(1996)Kavraki, Svestka, Latombe, and
  Overmars]{kavraki1996probabilistic}
Lydia Kavraki, Petr Svestka, J-C Latombe, and Mark Overmars.
\newblock Probabilistic roadmaps for path planning in high-dimensional
  configuration spaces.
\newblock \emph{IEEE Transactions on Robotics and Automation}, 12\penalty0
  (4):\penalty0 566--580, 1996.

\bibitem[LaValle(1998)]{lavalle1998rapidly}
Steven LaValle.
\newblock Rapidly-exploring random trees: A new tool for path planning.
\newblock \emph{TR 98-11, Computer Science Department, Iowa State University},
  1998.

\bibitem[LaValle and Kuffner~Jr(2001)]{lavalle2001randomized}
Steven~M LaValle and James~J Kuffner~Jr.
\newblock Randomized kinodynamic planning.
\newblock \emph{The International Journal of Robotics Research}, 20\penalty0
  (5):\penalty0 378--400, 2001.

\bibitem[Leomanni et~al.(2022)Leomanni, Costante, and
  Ferrante]{leomanni2022time}
Mirko Leomanni, Gabriele Costante, and Francesco Ferrante.
\newblock Time-optimal control of a multidimensional integrator chain with
  applications.
\newblock \emph{IEEE Control Systems Letters}, 6:\penalty0 2371--2376, 2022.

\bibitem[Li et~al.(2016)Li, Littlefield, and Bekris]{li2016asymptotically}
Yanbo Li, Zakary Littlefield, and Kostas~E Bekris.
\newblock Asymptotically optimal sampling-based kinodynamic planning.
\newblock \emph{The International Journal of Robotics Research}, 35\penalty0
  (5):\penalty0 528--564, 2016.

\bibitem[Lipp and Boyd(2014)]{lipp2014minimum}
Thomas Lipp and Stephen Boyd.
\newblock Minimum-time speed optimisation over a fixed path.
\newblock \emph{International Journal of Control}, 87\penalty0 (6):\penalty0
  1297--1311, 2014.

\bibitem[Liu et~al.(2017)Liu, Watterson, Mohta, Sun, Bhattacharya, Taylor, and
  Kumar]{liu2017planning}
Sikang Liu, Michael Watterson, Kartik Mohta, Ke~Sun, Subhrajit Bhattacharya,
  Camillo~J Taylor, and Vijay Kumar.
\newblock Planning dynamically feasible trajectories for quadrotors using safe
  flight corridors in 3-{D} complex environments.
\newblock \emph{IEEE Robotics and Automation Letters}, 2\penalty0 (3):\penalty0
  1688--1695, 2017.

\bibitem[Malyuta et~al.(2022)Malyuta, Reynolds, Szmuk, Lew, Bonalli, Pavone,
  and A{\c{c}}{\i}kme{\c{s}}e]{malyuta2022convex}
Danylo Malyuta, Taylor~P Reynolds, Michael Szmuk, Thomas Lew, Riccardo Bonalli,
  Marco Pavone, and Beh{\c{c}}et A{\c{c}}{\i}kme{\c{s}}e.
\newblock Convex optimization for trajectory generation: A tutorial on
  generating dynamically feasible trajectories reliably and efficiently.
\newblock \emph{IEEE Control Systems Magazine}, 42\penalty0 (5):\penalty0
  40--113, 2022.

\bibitem[Marcucci(2024)]{marcucci2024graphs}
Tobia Marcucci.
\newblock \emph{Graphs of Convex Sets with Applications to Optimal Control and
  Motion Planning}.
\newblock PhD thesis, Massachusetts Institute of Technology, 2024.

\bibitem[Marcucci et~al.(2023)Marcucci, Petersen, von Wrangel, and
  Tedrake]{marcucci2023motion}
Tobia Marcucci, Mark Petersen, David von Wrangel, and Russ Tedrake.
\newblock Motion planning around obstacles with convex optimization.
\newblock \emph{Science robotics}, 8\penalty0 (84):\penalty0 eadf7843, 2023.

\bibitem[Marcucci et~al.(2024{\natexlab{a}})Marcucci, Nobel, Tedrake, and
  Boyd]{marcucci2024fast}
Tobia Marcucci, Parth Nobel, Russ Tedrake, and Stephen Boyd.
\newblock Fast path planning through large collections of safe boxes.
\newblock \emph{IEEE Transactions on Robotics}, 40:\penalty0 3795--3811,
  2024{\natexlab{a}}.

\bibitem[Marcucci et~al.(2024{\natexlab{b}})Marcucci, Umenberger, Parrilo, and
  Tedrake]{marcucci2024shortest}
Tobia Marcucci, Jack Umenberger, Pablo Parrilo, and Russ Tedrake.
\newblock Shortest paths in graphs of convex sets.
\newblock \emph{SIAM Journal on Optimization}, 34\penalty0 (1):\penalty0
  507--532, 2024{\natexlab{b}}.

\bibitem[Morozov et~al.(2024)Morozov, Marcucci, Amice, Graesdal, Bosworth,
  Parrilo, and Tedrake]{morozov2024multi}
Savva Morozov, Tobia Marcucci, Alexandre Amice, Bernhard~Paus Graesdal, Rohan
  Bosworth, Pablo~A Parrilo, and Russ Tedrake.
\newblock Multi-query shortest-path problem in graphs of convex sets.
\newblock \emph{arXiv preprint arXiv:2409.19543}, 2024.

\bibitem[Pan and Manocha(2012)]{pan2012gpu}
Jia Pan and Dinesh Manocha.
\newblock {GPU}-based parallel collision detection for fast motion planning.
\newblock \emph{The International Journal of Robotics Research}, 31\penalty0
  (2):\penalty0 187--200, 2012.

\bibitem[Parrilo(2003)]{parrilo2003semidefinite}
Pablo~A Parrilo.
\newblock Semidefinite programming relaxations for semialgebraic problems.
\newblock \emph{Mathematical programming}, 96:\penalty0 293--320, 2003.

\bibitem[Schulman et~al.(2014)Schulman, Duan, Ho, Lee, Awwal, Bradlow, Pan,
  Patil, Goldberg, and Abbeel]{schulman2014motion}
John Schulman, Yan Duan, Jonathan Ho, Alex Lee, Ibrahim Awwal, Henry Bradlow,
  Jia Pan, Sachin Patil, Ken Goldberg, and Pieter Abbeel.
\newblock Motion planning with sequential convex optimization and convex
  collision checking.
\newblock \emph{The International Journal of Robotics Research}, 33\penalty0
  (9):\penalty0 1251--1270, 2014.

\bibitem[Shen et~al.(2017)Shen, Diamond, Udell, Gu, and
  Boyd]{shen2017disciplined}
Xinyue Shen, Steven Diamond, Madeleine Udell, Yuantao Gu, and Stephen Boyd.
\newblock Disciplined multi-convex programming.
\newblock In \emph{29th Chinese Control and Decision Conference}, pages
  895--900. IEEE, 2017.

\bibitem[Sundaralingam et~al.(2023)Sundaralingam, Hari, Fishman, Garrett,
  Van~Wyk, Blukis, Millane, Oleynikova, Handa, Ramos,
  et~al.]{sundaralingam2023curobo}
Balakumar Sundaralingam, Siva Kumar~Sastry Hari, Adam Fishman, Caelan Garrett,
  Karl Van~Wyk, Valts Blukis, Alexander Millane, Helen Oleynikova, Ankur Handa,
  Fabio Ramos, et~al.
\newblock Curobo: Parallelized collision-free robot motion generation.
\newblock In \emph{IEEE International Conference on Robotics and Automation},
  pages 8112--8119. IEEE, 2023.

\bibitem[Tedrake and the Drake Development~Team(2019)]{tedrake2019drake}
Russ Tedrake and the Drake Development~Team.
\newblock Drake: Model-based design and verification for robotics, 2019.

\bibitem[Tedrake et~al.(2010)Tedrake, Manchester, Tobenkin, and
  Roberts]{tedrake2010lqr}
Russ Tedrake, Ian Manchester, Mark Tobenkin, and John Roberts.
\newblock {LQR}-trees: Feedback motion planning via sums-of-squares
  verification.
\newblock \emph{The International Journal of Robotics Research}, 29\penalty0
  (8):\penalty0 1038--1052, 2010.

\bibitem[Thomason et~al.(2024)Thomason, Kingston, and
  Kavraki]{thomason2024motions}
Wil Thomason, Zachary Kingston, and Lydia~E Kavraki.
\newblock Motions in microseconds via vectorized sampling-based planning.
\newblock In \emph{IEEE International Conference on Robotics and Automation},
  pages 8749--8756. IEEE, 2024.

\bibitem[Toussaint(2014)]{toussaint2014newton}
Marc Toussaint.
\newblock Newton methods for k-order markov constrained motion problems.
\newblock \emph{arXiv preprint arXiv:1407.0414}, 2014.

\bibitem[Verscheure et~al.(2009)Verscheure, Demeulenaere, Swevers, De~Schutter,
  and Diehl]{verscheure2009time}
Diederik Verscheure, Bram Demeulenaere, Jan Swevers, Joris De~Schutter, and
  Moritz Diehl.
\newblock Time-optimal path tracking for robots: A convex optimization
  approach.
\newblock \emph{IEEE Transactions on Automatic Control}, 54\penalty0
  (10):\penalty0 2318--2327, 2009.

\bibitem[W{\"a}chter and Biegler(2006)]{wachter2006implementation}
Andreas W{\"a}chter and Lorenz~T Biegler.
\newblock On the implementation of an interior-point filter line-search
  algorithm for large-scale nonlinear programming.
\newblock \emph{Mathematical Programming}, 106:\penalty0 25--57, 2006.

\bibitem[Wang et~al.(2024)Wang, Wang, Wang, Ji, Han, Wu, Jin, Gao, Xu, and
  Gao]{wang2024fast}
Qianhao Wang, Zhepei Wang, Mingyang Wang, Jialin Ji, Zhichao Han, Tianyue Wu,
  Rui Jin, Yuman Gao, Chao Xu, and Fei Gao.
\newblock Fast iterative region inflation for computing large 2-{D}/3-{D}
  convex regions of obstacle-free space.
\newblock \emph{arXiv preprint arXiv:2403.02977}, 2024.

\bibitem[Wang and Boyd(2009)]{wang2009fast}
Yang Wang and Stephen Boyd.
\newblock Fast model predictive control using online optimization.
\newblock \emph{IEEE Transactions on Control Systems Technology}, 18\penalty0
  (2):\penalty0 267--278, 2009.

\bibitem[Werner et~al.(2024{\natexlab{a}})Werner, Amice, Marcucci, Rus, and
  Tedrake]{werner2024approximating}
Peter Werner, Alexandre Amice, Tobia Marcucci, Daniela Rus, and Russ Tedrake.
\newblock Approximating robot configuration spaces with few convex sets using
  clique covers of visibility graphs.
\newblock In \emph{IEEE International Conference on Robotics and Automation},
  pages 10359--10365. IEEE, 2024{\natexlab{a}}.

\bibitem[Werner et~al.(2024{\natexlab{b}})Werner, Cohn, Jiang, Seyde,
  Simchowitz, Tedrake, and Rus]{werner2024faster}
Peter Werner, Thomas Cohn, Rebecca~H Jiang, Tim Seyde, Max Simchowitz, Russ
  Tedrake, and Daniela Rus.
\newblock Faster algorithms for growing collision-free convex polytopes in
  robot configuration space.
\newblock \emph{arXiv preprint arXiv:2410.12649}, 2024{\natexlab{b}}.

\bibitem[Werner et~al.(2025)Werner, Cheng, Stewart, Tedrake, and
  Rus]{werner2025superfast}
Peter Werner, Richard Cheng, Tom Stewart, Russ Tedrake, and Daniela Rus.
\newblock Superfast configuration-space convex set computation on {GPU}s for
  online motion planning.
\newblock \emph{arXiv preprint arXiv:2504.10783}, 2025.

\bibitem[Wu et~al.(2024)Wu, Spasojevic, Chaudhari, and Kumar]{wu2024optimal}
Yuwei Wu, Igor Spasojevic, Pratik Chaudhari, and Vijay Kumar.
\newblock Optimal convex cover as collision-free space approximation for
  trajectory generation.
\newblock \emph{arXiv preprint arXiv:2406.09631}, 2024.

\bibitem[Yang et~al.(2025)Yang, Marcucci, Parrilo, and Tedrake]{yang2025new}
Lujie Yang, Tobia Marcucci, Pablo~A Parrilo, and Russ Tedrake.
\newblock A new semidefinite relaxation for linear and piecewise-affine optimal
  control with time scaling.
\newblock \emph{arXiv preprint arXiv:2504.13170}, 2025.

\bibitem[Zucker et~al.(2013)Zucker, Ratliff, Dragan, Pivtoraiko, Klingensmith,
  Dellin, Bagnell, and Srinivasa]{zucker2013chomp}
Matt Zucker, Nathan Ratliff, Anca~D Dragan, Mihail Pivtoraiko, Matthew
  Klingensmith, Christopher~M Dellin, J~Andrew Bagnell, and Siddhartha~S
  Srinivasa.
\newblock {CHOMP}: Covariant hamiltonian optimization for motion planning.
\newblock \emph{The International Journal of Robotics Research}, 32\penalty0
  (9-10):\penalty0 1164--1193, 2013.

\end{thebibliography}
